\newcommand{\deepmindstyle}{}  

\ifx\deepmindstyle\undefined
\documentclass{article} 
\PassOptionsToPackage{authoryear, sort&compress, round}{natbib}
\usepackage{neurips_2019}
\usepackage[utf8]{inputenc} 
\usepackage[T1]{fontenc}    
\usepackage{hyperref}       
\usepackage{booktabs}       
\usepackage{amsfonts}       
\usepackage{nicefrac}       
\usepackage{microtype}      
\usepackage{url}
\usepackage[dvipsnames,table]{xcolor}
\else
\documentclass[10pt, a4paper]{deepmind}
\usepackage[authoryear, sort&compress, round]{natbib}
\fi

\usepackage{kantlipsum, lipsum}
\usepackage{amsmath, latexsym}
\usepackage{amsfonts}
\usepackage{mathtools}
\usepackage{ntheorem}
\usepackage{dsfont}
\usepackage[dvipsnames,table]{xcolor}
\usepackage[colorinlistoftodos]{todonotes}
\usepackage{booktabs}
\usepackage{xfrac}
\usepackage{bbm}

\usepackage{algorithm}
\usepackage{algorithmicx}

\usepackage[most]{tcolorbox}
\usepackage{xparse}
\usepackage{lipsum}
\usepackage{changepage}
\usepackage{enumitem}

\newcommand{\squishlist}{
   \begin{list}{$\bullet$}
    { \setlength{\itemsep}{0pt}      \setlength{\parsep}{3pt}
      \setlength{\topsep}{3pt}       \setlength{\partopsep}{0pt}
      \setlength{\leftmargin}{1.5em} \setlength{\labelwidth}{1em}
      \setlength{\labelsep}{0.5em} } }

\newcommand{\squishlisttwo}{
   \begin{list}{$\bullet$}
    { \setlength{\itemsep}{0pt}    \setlength{\parsep}{0pt}
      \setlength{\topsep}{0pt}     \setlength{\partopsep}{0pt}
      \setlength{\leftmargin}{2em} \setlength{\labelwidth}{1.5em}
      \setlength{\labelsep}{0.5em} } }

\newcommand{\squishend}{
    \end{list}  }

\newtheorem{thm}{Theorem}[section]

\newenvironment{proof}{{\bf Proof.}}{}

\newcommand{\grad}{\nabla}
\newcommand{\proj}{\mbox{\texttt{Proj}}}

\newcommand{\two}{(2)}

\newcommand{\sign}{\mbox{sign}}

\newcommand{\E}{\mathbb{E}}

\DeclareMathAlphabet{\mathpzc}{OT1}{pzc}{m}{n}

\newcommand{\R}{\mathbb{R}}

\DeclareMathOperator*{\argmax}{argmax}
\DeclareMathOperator*{\argmin}{argmin}
\newcommand{\mnist}{\textsc{Mnist}\xspace}
\newcommand{\cifar}{\textsc{Cifar-10}\xspace}
\newcommand{\svhn}{\textsc{Svhn}\xspace}
\newcommand{\imagenet}{\textsc{ImageNet}\xspace}
\newcommand{\wideresnet}{WideResNet\xspace}
\newcommand{\resnet}{ResNet\xspace}
\newcommand{\multitargeted}{\texttt{MultiTargeted}\xspace}
\newcommand{\linf}{\ensuremath{\ell_\infty}\xspace}
\newcommand{\ltwo}{\ensuremath{\ell_2}\xspace}
\newcommand{\lone}{\ensuremath{\ell_1}\xspace}

\DeclareMathOperator*{\maximize}{max}

\definecolor{Highlight}{gray}{0.9}
\newcommand{\Tstrut}{\rule{0pt}{2.6ex}}
\newcommand{\Bstrut}{\rule[-0.9ex]{0pt}{0pt}}
\newcommand{\TBstrut}{\Tstrut\Bstrut}

\usepackage{xspace}
\usepackage{algpseudocode}
\usepackage{wrapfig}
\usepackage{subcaption}
\usepackage{multirow}
\graphicspath{{figures/}}

\title{An Alternative Surrogate Loss for PGD-based Adversarial Testing}

\newcommand{\abtracttext}{
Adversarial testing methods based on Projected Gradient Descent (PGD) are widely used for searching norm-bounded perturbations that cause the inputs of neural networks to be misclassified.
This paper takes a deeper look at these methods and explains the effect of different hyperparameters (i.e., optimizer, step size and surrogate loss).
We introduce the concept of \multitargeted testing, which makes clever use of alternative surrogate losses, and explain when and how \multitargeted is guaranteed to find optimal perturbations.
Finally, we demonstrate that \multitargeted outperforms more sophisticated methods and often requires less iterative steps than other variants of PGD found in the literature.
Notably, \multitargeted ranks first on MadryLab's white-box \mnist and \cifar leaderboards, reducing the accuracy of their \mnist model to 88.36\% (with \linf perturbations of $\epsilon = 0.3$) and the accuracy of their \cifar model to 44.03\% (at $\epsilon = 8/255$).
\multitargeted also ranks first on the TRADES leaderboard reducing the accuracy of their \cifar model to 53.07\% (with \linf perturbations of $\epsilon = 0.031$).
}

\ifx\deepmindstyle\undefined
\author{%
  Sven Gowal \\
  DeepMind \\
  \texttt{sgowal@google.com}
  \And
  Jonathan Uesato \\
  \And
  Chongli Qin \\
  \And
  Po-Sen Huang \\
  \AND
  Timothy Mann \\
  \And
  Pushmeet Kohli \\
}

\else
\renewcommand{\shorttitle}{An Alternative Surrogate Loss for PGD-based Adversarial Testing}

\correspondingauthor{sgowal@google.com}

\reportnumber{001} 

\author[1]{Sven Gowal}
\author[1]{Jonathan Uesato}
\author[1]{Chongli Qin}
\author[1]{Po-Sen Huang}
\author[1]{Timothy Mann}
\author[1]{Pushmeet Kohli}

\affil[1]{DeepMind}

\begin{abstract}
\abtracttext
\end{abstract}
\fi

\begin{document}
\maketitle

\ifx\deepmindstyle\undefined
\begin{abstract}
\abtracttext
\end{abstract}
\else
\balance
\fi

\section{Introduction}

Despite the successes of deep learning \citep{goodfellow_deep_2016}, it is well-known that neural networks are not intrinsically robust.
In particular, it has been shown that the addition of small but carefully chosen deviations to the input, called adversarial perturbations, can cause the neural network to make incorrect predictions with high confidence \citep{carlini_adversarial_2017,carlini_towards_2017,goodfellow_explaining_2014,kurakin_adversarial_2016,szegedy_intriguing_2013}. The sensitivity of neural network models necessitates the development of methods that can either systematically find such failure cases, or declare that no such failure case exists with high-confidence. Starting with \citet{szegedy_intriguing_2013}, there has been a lot of work on understanding and generating adversarial perturbations \citep{carlini_towards_2017,athalye_synthesizing_2017}, and on building models that are robust to such perturbations \citep{goodfellow_explaining_2014,papernot_distillation_2015,madry_towards_2017,kannan_adversarial_2018,xie_feature_2018}.
Unfortunately, many strategies proposed in the literature target failure cases found through specific adversaries, and as such they are easily broken under different adversaries \citep{uesato_adversarial_2018,athalye_obfuscated_2018}.
Again, this phenomenon highlights the importance of understanding the limitations of different adversarial techniques when it comes to finding when models fail.

Instead of resorting to methods that directly focus on improving robustness to specific attacks, \citet{madry_towards_2017} formulate a saddle point problem whose goal is to find model parameters $\theta$ that minimize the adversarial risk:
\begin{equation}
\E_{(x,y) \sim \mathcal{D}} \left[ \maximize_{\xi \in \mathcal{S}(x)} L(f_\theta(\xi), y) \right]
\label{eq:adversarial_risk}
\end{equation}
where $\mathcal{D}$ is a data distribution over pairs of examples $x$ and corresponding labels $y$, $f_\theta$ is a model parametrized by $\theta$, $L$ is a suitable loss function (such as the $0-1$ loss in the context of classification tasks), and $\mathcal{S}(x)$ defines the set of allowed perturbations (i.e., the adversarial input set or threat model).
The aim of rigorous testing is to accurately estimate this adversarial risk.
As such, finding the worst-case input (or optimal adversarial example) $\xi^\star \in \mathcal{S}(x)$ is a key consideration for both training and testing models.
Finding sub-optimal inputs $\hat{\xi}$ with $L(f_\theta(\hat{\xi}), y) \leq L(f_\theta(\xi^\star), y)$ results in computing a lower bound on the true adversarial risk (or in the context of a classification task, an upper bound on the true robust accuracy), which may give a false sense of security.
Hence, it is of practical importance to find the solution of the inner maximization problem in Equation~\eqref{eq:adversarial_risk} as efficiently and as accurately as possible.

\subsection{Methods for Adversarial Testing}\label{sec:related_work}
Several methods (also known as ``attacks'') have been proposed to find adversarial examples (and effectively solve the inner maximization problem in Equation~\eqref{eq:adversarial_risk}).
\citet{goodfellow_explaining_2014} proposed one of the earliest method specific to \linf-bounded perturbations.
It is known as the Fast Gradient Sign Method (FGSM):
it replaces the impractical $0-1$ loss $L$ with the cross-entropy loss $\hat{L}$ and computes an adversarial example $\hat{\xi}$ as
\begin{equation}
x + \epsilon \sign \left( \grad_x \hat{L}(f_\theta(x), y) \right) .
\end{equation}
A more powerful adversary is the multi-step variant FGSM$^K$, which essentially performs $K$ projected gradient steps on the surrogate loss $\hat{L}$~\citep{kurakin_adversarial_2016} to find $\hat{\xi} = \xi^{(K)}$ with
\begin{equation}
\xi^{(k+1)} \gets \proj_{\mathcal{S}(x)} \left( \xi^{(k)} + \alpha \sign \left( \grad_{\xi^{(k)}} \hat{L}(f_\theta(\xi^{(k)}), y) \right) \right)
\end{equation}
where $\xi^{(0)}$ is typically chosen at random within $S(x)$, and \proj~is a projection operator, for example:
\begin{equation}
\proj_{\mathcal{S}(x)}(\xi) = \argmin_{\xi' \in \mathcal{S}(x)} \| \xi' - \xi \|_2 .
\end{equation}
Another popular method stems from \cite{carlini_towards_2017} who extend the formulation of \cite{szegedy_intriguing_2013}.
They investigate how different optimization methods and different losses $\hat{L}$ affect the quality of the adversarial examples.
Crucially, and similarly to \cite{uesato_adversarial_2018}, they also propose to use the Adam~\citep{kingma_adam:_2014} optimizer rather than regular gradient descent.
Finally, since gradient-based method are sensitive to initial conditions, it is not uncommon to repeat the optimization procedure many times with different initializations $\xi^{(0)}$ (i.e., multiple restarts).
All above methods are approximate and one generally hopes to find an adversarial example $\hat{\xi}$ such that $L(f_\theta(\hat{\xi}, y) = L(f_\theta(\xi^\star, y)$.
A ``stronger'' attack finds such adversarial examples more often than a ``weaker'' attack for $(x,y) \sim \mathcal{D}$.

More recently, \cite{zheng_distributionally_2018} developed the Distributionally Adversarial Attack (DAA) which performs optimization on the space of potential data distributions and ranked first place on MadryLab's white-box leaderboards (both on \mnist\footnote{\url{https://github.com/MadryLab/mnist_challenge}} and \cifar\footnote{\url{https://github.com/MadryLab/
cifar10_challenge}}).
This attack was overtaken by the Interval Attack~\citep{wang_enhancing_2019} and the Fast Adaptive Boundary Attack (FAB,~\citealp{croce2019minimally}) on the \mnist and \cifar leaderboards, respectively.
\cite{brendel2019accurate} extended their decision-based boundary attack~\citep{brendel2017decision} with an emphasis on efficiency (which is not the focus on this work).
Note that there exists many more attacks~\citep{papernot_limitations_2016,moosavi_deepfool_2016,chen2018ead} that are either weaker or optimize norms other than \linf.

\subsection{Revisiting PGD}
The increasing complexity and sophistication of methods for adversarial testing makes them difficult to adapt to different models, architectures and application areas.
This motivated us to revisit Projected Gradient Descent (PGD) as a means to solve the maximization problem in Equation~\eqref{eq:adversarial_risk}.
We focus on robustness against \linf-bounded attacks (i.e., $\mathcal{S}(x) = \{ \xi ~:~ \| \xi - x \|_\infty \leq \epsilon \}$ where $\epsilon$ is an a-priori defined constant).
We expect the analysis performed here to translate well to $\lone$ and $\ltwo$-bounded attacks and reserve the evaluation of these threat models for future work. 

\noindent The contributions of this paper are as follows:
\squishlist
\item We create a guide to PGD, explaining the different variants and highlighting their use in practice. We provide a set of hyperparameters (i.e., optimizer, step size and surrogate loss) with which PGD does well across a wide range of models. We demonstrate that \emph{regular PGD} (as opposed to the \multitargeted version described below) can be tuned such that it would rank first on MadryLab's white-box \mnist leaderboard, reducing the accuracy of their model to 88.21\% (with \linf perturbations of $\epsilon = 0.3$).
\item We introduce the concept of a \multitargeted attack and demonstrate its effectiveness by comparing it with other regular PGD-based attacks under the same computational budget. \multitargeted is a PGD-based untargeted attack with an alternative surrogate loss.
\item We compare \multitargeted with other state-of-the-art methods.
We show that \multitargeted ranks first on MadryLab's white-box \mnist leaderboard, reducing the accuracy of their model to 88.36\% (with \linf perturbations of $\epsilon = 0.3$).
We also found \multitargeted to be particularly effective on models trained on \cifar and \imagenet. It ranks first on MadryLab's white-box \cifar leaderboard, reducing the accuracy of their model to 44.03\% (with \linf perturbations of $\epsilon = 8/255$).
\multitargeted also ranks first on the TRADES leaderboard reducing the accuracy of their \cifar model to 53.07\% (with \linf perturbations of $\epsilon = 0.031$).
\squishend

\section{PGD-based rigorous testing}\label{sec:preliminaries}

As explained in Section~\ref{sec:related_work}, there exists many attack variants.
In this paper, we focus on variants of PGD.
A PGD-based attack typically executes the following procedure:
\begin{algorithm}[h]
\footnotesize
\caption{White-box PGD-based attack}
\begin{algorithmic}[1]
\Require A nominal input $x$ and label $y$, a model $f_\theta$ and an adversarial input set $\mathcal{S}(x)$. $\alpha^{(k)}$ is the step-size at iteration $k$ and $\texttt{Opt}$ is an optimizer (e.g., \texttt{Adam}).
\Ensure Possible attack $\hat{\xi} \in \mathcal{S}(x)$
\State $\hat{\xi} \gets x$
\For{$r \in \{1, \ldots, N_\textrm{r}\}$} \Comment{Repeat the attack process $N_\textrm{r}$ times} \label{line:start}
  \State Initialize surrogate loss $\hat{L}^{(r)}$ \Comment{The surrogate loss may be different across restarts (see Sec.~\ref{sec:method})} \label{line:end}
  \State Initialize optimizer \texttt{Opt} \Comment{The optimizer ingests gradients and computes an update}
  \State $\xi^{(0)} \gets \texttt{SampleFrom}(x, \mathcal{S})$ \Comment{Samples a random input that respects the threat model}
  \For{$k \in \{1, \ldots, K\}$} \Comment{K corresponds to the number of optimization steps}
    \State $\xi^{(k)} \gets \proj_{\mathcal{S}(x)} \left( \xi^{(k-1)} + \alpha^{(k)} \texttt{Opt} \left( \grad_{\xi^{(k-1)}} \hat{L}^{(r)}(f_\theta(\xi^{(k-1)}), y) \right) \right)$
    \If{$L(f_\theta(\xi^{(k)}), y) > L(f_\theta(\hat{\xi}, y)$} \Comment{$L$ is the loss of interest (e.g., $0-1$ loss)}
      \State $\hat{\xi} \gets \xi^{(k)}$ \Comment{If $L$ is the $0-1$ loss, the procedure can terminate early}
    \EndIf
  \EndFor
\EndFor
\end{algorithmic}
\label{alg:pgd}
\end{algorithm}

\noindent Algorithm~\ref{alg:pgd} highlights a few design choices.
Beyond the number of restarts $N_\textrm{r}$ and the number of PGD steps $K$\footnote{We are actually performing projected gradient ascent steps, but will keep using the PGD acronym.}, one must decide which optimizer \texttt{Opt}, which sub-differentiable surrogate loss $\hat{L}$, and which step size schedule (or learning rate) $\alpha^{(k)}$ to use.
Apart from sometimes varying the number of steps $K$, recent works often use a single set of fixed parameters to provide baseline results (irrespective of the network architecture).
While none of these choices are necessarily trivial, there exists reasonable options (some of which are explained in details by~\citealp{carlini2019evaluating}).
Here is a \emph{non-exhaustive} list.

\paragraph{Optimizer.}
Since the work from \cite{kurakin_adversarial_2016} on \linf-norm bounded perturbations, the \sign~function has been most commonly used. More recently, \cite{carlini_towards_2017} and subsequently \cite{dong2018boosting} proposed to use the Adam~\citep{kingma_adam:_2014} and Momentum~\citep{polyak1964some} optimizers, respectively.
Most works limit themselves to zeroth-order and first-order methods for efficiency, but in principle higher order methods are applicable (especially when the dimensionality of $x$ is small).
In general, we found that the Adam optimizer produced the most consistent results (similar conclusions were reached by \citealp{uesato_adversarial_2018}) and was less sensitive to the step size. There were, however, cases where the \sign~function was more efficient (e.g., when models used activations that saturate such as the sigmoid function).
For clarity in the algorithm description, we categorize the \sign~function as being part of the optimizer.
As such, it also worth mentioning that any other normalization of the gradient is possible (e.g., \ltwo-normalization when dividing by $\| \grad_{x^{(k-1)}} \hat{L}^{(r)}(f_\theta(x^{(k-1)}), y) \|_2$).

\paragraph{Surrogate loss.} The most common (and closely related) surrogate losses used in the literature are
\begin{align}
\hat{L}(z, y) &= - z_y + \log \left( \sum_{i = 1}^C e^{z_i} \right) \quad & \textrm{cross-entropy loss}\\
\hat{L}(z, y) &= - z_y + \max_{i \neq y} z_i & \textrm{margin loss} \label{eq:margin_loss}
\end{align}
where $z = f_\theta(x) \in \R^C$ are logits (and $C$ is the number of classes).
For most models, both losses tend to perform equally well (which is expected given their close relationship).
In some cases, we found the margin loss to be slightly better (as reported by \citealp{carlini_towards_2017}).
\cite{carlini_towards_2017} also propose a list of seven losses that operate in both probability and logit spaces (loss 1 is the cross-entropy loss, while loss 6 is closely related to the margin loss).

\paragraph{Step size.} The step size (or learning rate) can have a significant influence on the success rate of the attack.
It is often easier to tune the step size when using normalized gradients, and most works limit themselves to well established values.
We found that using a schedule for the step size can help alleviate some the weakness induced by gradient obfuscation~\citep{uesato_adversarial_2018,carlini_adversarial_2017} while maintaining a small number of PGD steps.
The right schedule depends on the choice of optimizer, model and dataset studied.

\paragraph{Trade-off between number of restarts and number of PGD steps.}
When compute is not limited, it is always preferable to increase the number of restarts $N_\textrm{r}$ and the number of PGD steps $K$.
However, with limited compute, the relationship between $N_\textrm{r}$ and reduce $K$ is unclear.
As a rule of thumb, we suggest tuning the optimizer, surrogate loss and step size with one restart and then increase the number of restarts as much as possible for the available compute.

For classification tasks, a popular choice is FGSM$^K$ which can be implemented by setting the optimizer to the \sign~function, by using cross-entropy as surrogate loss, and by fixing $\alpha^{(k)}$ to a fixed constant (typically around $\epsilon / 10$).
While FGSM$^K$ often gives reasonable results, we found that these parameters were \emph{always} sub-optimal across all the datasets (i.e., \mnist, \cifar, \svhn, \imagenet) and models (e.g., \wideresnet, \resnet and other standard convolutional networks) we tested over the past year.
Our usual starting point for tuning PGD-based attacks is to set the optimizer to Adam, set the step size to $0.1$ (with $10\times$ decay at $K/2$ and $3K/4$ steps) and use the margin loss.
Again, it is really important to understand that, while these values tend to work well in practice for commonly seen models and datasets, they need to be tuned for each case specially.

\section{Method}\label{sec:method}

\multitargeted is an instantiation of Algorithm~\ref{alg:pgd} where we set the surrogate loss to
\begin{align}
\hat{L}^{(r)}(z, y) &= - z_y + z_{\mathcal{T}_{r \bmod |\mathcal{T}|}} \textrm{ with } \mathcal{T} = \{1, \ldots, C\} \setminus \{y\} \label{eq:multitargeted_loss}
\end{align}
In other words, at each restart, we pick a new target class $t \in \mathcal{T}$ where the set $\mathcal{T}$ contains all classes other than $y$.
Rewriting lines~\ref{line:start}-\ref{line:end} of Algorithm~\ref{alg:pgd} more programmatically, we obtain Algorithm~\ref{alg:multitargeted}.
\begin{algorithm}[ht]
\footnotesize
\caption{\multitargeted attack}
\begin{algorithmic}[1]
\setcounter{ALG@line}{1}\algrenewcommand\alglinenumber[1]{\footnotesize #1a:}%
  \State $N_\textrm{i} \gets \left\lfloor N_\textrm{r} / |\mathcal{T}| \right\rfloor$ \Comment{Keep the number of iterations consistent}
\setcounter{ALG@line}{1}\algrenewcommand\alglinenumber[1]{\footnotesize #1b:}%
\For{$i \in \{1, \ldots, N_\textrm{i} \}$}
\setcounter{ALG@line}{1}\algrenewcommand\alglinenumber[1]{\footnotesize #1c:}%
  \For{$j \in \{1, \ldots, |\mathcal{T}|\}$}
\algrenewcommand\alglinenumber[1]{\footnotesize #1:}%
    \State Use surrogate loss $\hat{L}^{(i |\mathcal{T}| + j + 1)}(z, y) = -z_y + z_t$ with $t = \mathcal{T}_j$
\algrenewcommand\alglinenumber[1]{}%
    \State ...
\setcounter{ALG@line}{10}\algrenewcommand\alglinenumber[1]{\footnotesize #1:}%
  \EndFor
\EndFor
\end{algorithmic}
\label{alg:multitargeted}
\end{algorithm}

\multitargeted uses $|\mathcal{T}|$ different surrogate losses (one for each class except the true class).
To keep the number of restarts comparable with other PGD-based attacks that use a single fixed surrogate loss, Algorithm~\ref{alg:multitargeted} uses less restarts per target class $t$ (i.e., $\lfloor N_\textrm{r} / \mathcal{T} \rfloor$).
To limit runtime complexity, it is possible to restrict the number of target classes by targeting only the top-$T$ classes:
\begin{align}
\mathcal{T} &= \left \{ t \quad\arrowvert\quad T > |\{ i \textrm{~such that~} i \neq y \textrm{ and } z_i > z_t \}|  \right \}
\end{align}
In Section~\ref{sec:results}, we will see that \multitargeted outperforms regular PGD-based attacks (any of the variants explained in Section~\ref{sec:preliminaries}).
As a results, it is often possible to lower the total number of PGD steps $K$, the number of inner restarts $N_\textrm{i}$ or the size of the target set $\mathcal{T}$ to have comparable (or even lower) runtime complexity while still outperforming regular PGD-based attacks.

\subsection{Analysis}\label{appendix:mt_analysis}
\label{sec:example}

\begin{figure*}[t]
\centering
\begin{subfigure}{0.45\textwidth}
\centering
\includegraphics[width=\linewidth]{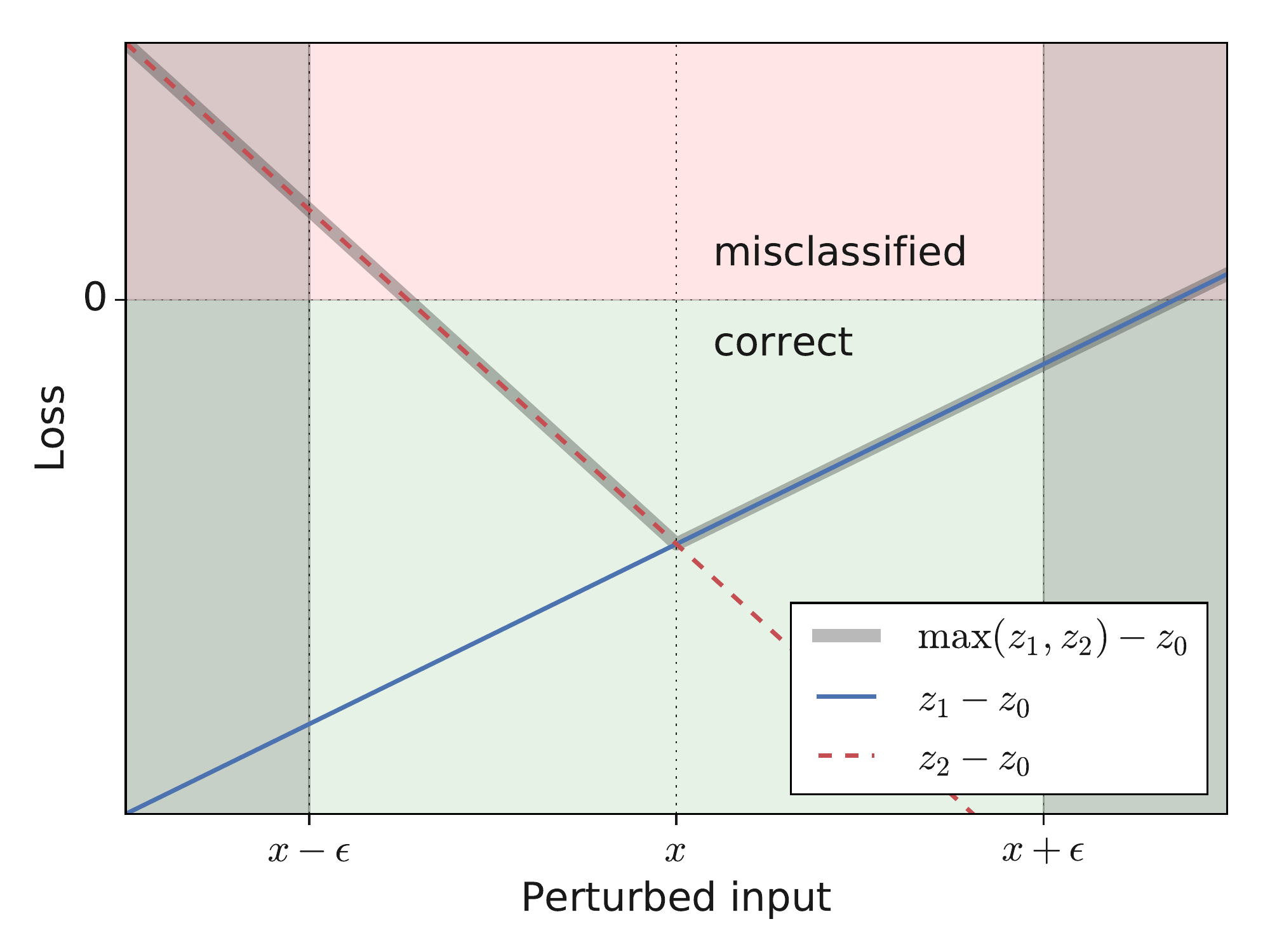}
\caption{\label{fig:motivating_example}}
\end{subfigure}
\hspace{1cm}
\begin{subfigure}{0.45\textwidth}
\centering
\includegraphics[width=\linewidth]{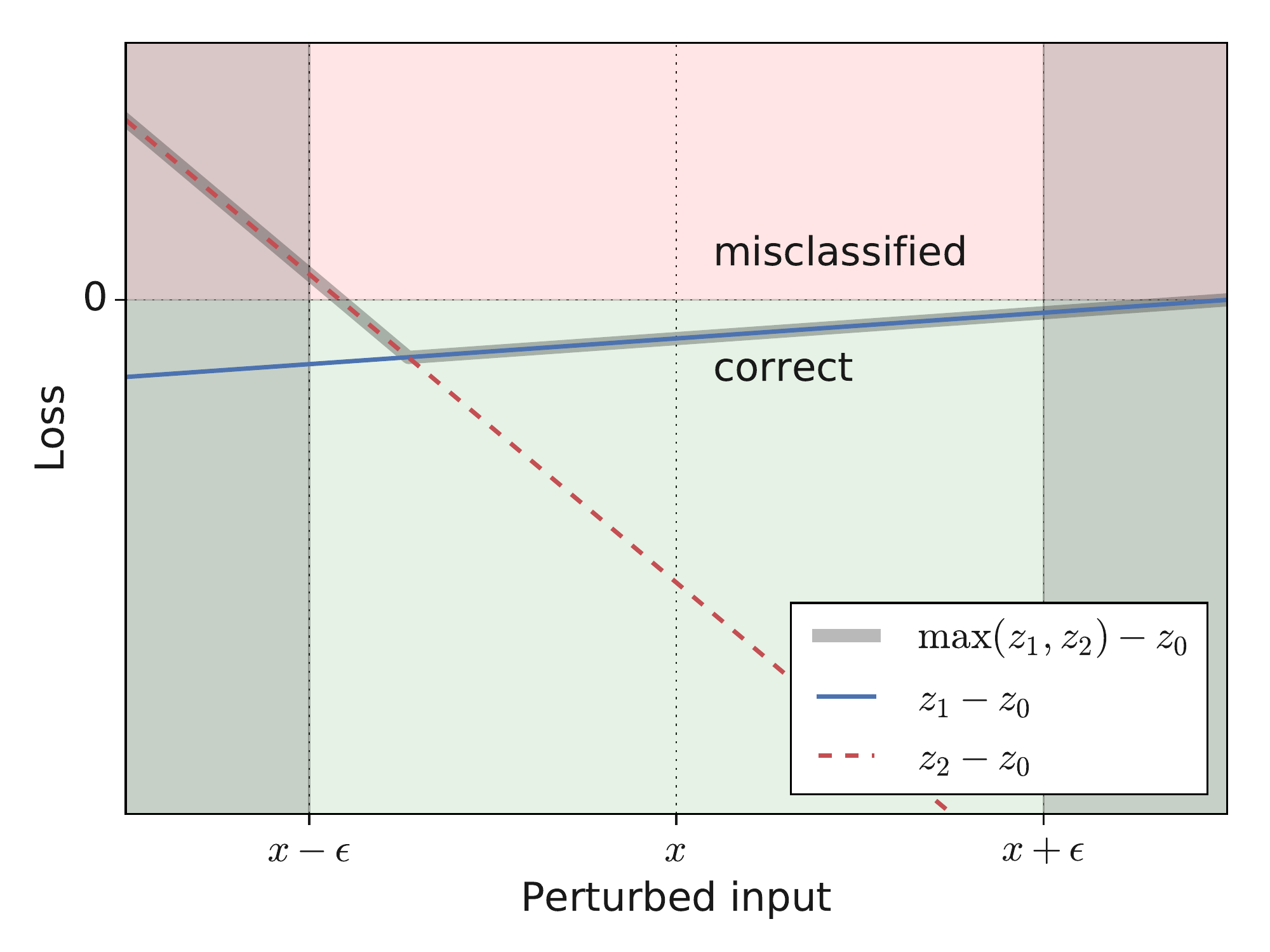}
\caption{\label{fig:motivating_example2}}
\end{subfigure}
\caption{Panel \subref{fig:motivating_example} shows an example motivating why \multitargeted can outperform other untargeted attacks. In this particular example a regular PGD-based attack with 2 restarts achieves 75\% success rate (as opposed to 100\% for \multitargeted). Panel \subref{fig:motivating_example2} shows a more extreme example. Here, a regular PGD-based attack with 2 restarts achieves 25\% success rate only. For both panels, the areas shaded in gray are outside the adversarial input set.}
\label{fig:examples}
\end{figure*}

\subsubsection{Toy example}
To motivate the effectiveness of the \multitargeted attack, we use the example depicted in Figure~\ref{fig:motivating_example}.
This example considers a one dimensional 3-way classification task with a linear classifier, hence, there are three output logits $z_0$, $z_1$ and $z_2$.
Without perturbation, a single dimensional input $x$ with true class $0$ is correctly classified (i.e., $z_0 > z_1$ and $z_0 > z_2$).
The y-axis represents the values taken by the margin loss $\max(z_1, z_2) - z_0$ (one of the surrogate losses used by regular PGD-based attacks), as well as the logit differences $z_1 - z_0$ and $z_2 - z_0$ (which are the surrogate losses used by \multitargeted) as a function of the perturbed input $x + \delta$.

The adversarial input set is $\mathcal{S}(x) = [x - \epsilon, x + \epsilon]$.
The perturbed input $x + \delta$ is correctly classified if and only if the logit differences are both negative.
Note how $z_1 - z_0$ cannot be positive (i.e., $z_1 > z_0$) for any $|\delta| < \epsilon$, to the contrary of $z_2 - z_0$ (i.e., $z_2 > z_0$).
As such, within the adversarial input set, $x + \delta$ can only be confused for class 2.
The success of a regular PGD-based attack that uses the margin loss depends on the initialization $\xi^{(0)}$.
When $\xi^{(0)}$ is initialized between $x$ and $x + \epsilon$ (this happens with 50\% chance if $\xi^{(0)}$ is uniformly sampled in $\mathcal{S}(x)$, the attack will converge to $\hat{\xi} = x + \epsilon$, which is correctly classified.
When $\xi^{(0)}$ is sampled between $x - \epsilon$ and $\epsilon$, the attack will converge to $\hat{\xi} = x - \epsilon$, which is misclassified.
Hence, with two restarts, when $\xi^{(0)}$ is chosen uniformly at random, the attack will succeed with probability $1 - (1 - 0.5)^2 = 75\%$.
However \multitargeted, which explicitly uses the logit differences as surrogate losses independently, will have a $100\%$ success rate (when using two restarts).
Figure~\ref{fig:motivating_example2} shows a more extreme example where the logit difference $z_2 - z_0$ overlaps minimally with the margin loss.
In this example, with two restarts, a regular PGD-based attack achieves 25\% success rate only.

\subsubsection{Linear models}
We can repeat the experiment illustrated by these examples.
We set the nominal input $x$ to zero and use a maximum perturbation radius $\epsilon$ of one.
Assuming a linear classifier $f(x) = Wx + b$, we uniformly at random sample weights $W$ and biases $b$ within the $[-1, 1]$ interval.
We assign the true class $y$ to be index of the highest logit at $x = 0$ (i.e., $\argmax_i b_i$).
We consider an input $x$ to be attackable if there exists a perturbation such that a logit that does not belong to the true class is the highest (i.e., $\exists \delta \in [-\epsilon, \epsilon], i \neq y$ such that $f(x + \delta)_i > f(x + \delta)_y$).
Given the linearity of the classifier and, as established above, if an input is attackable, \multitargeted will succeed in finding a misclassified input $100\%$ of the time (if the number of restarts is at least equal to the number of classes minus one).
Repeating the experiment a million times, a regular PGD-based attack succeed only $96.16\%$ of the time when the number of classes is three and the number of restarts is two.
Slicing this result, we also notice that when the number of confusing classes (classes different from the true class for which there exists a perturbed input that is classified as such, i.e., $|\{i | \exists \delta \textrm{ with } f(x + \delta)_i > f(x + \delta)_y\}|$) is two, it succeeds $100\%$ of the time (as expected). When the number of confusing classes is one, it succeeds $92.39\%$ of the time.
Figure~\ref{fig:numerical_example} shows how these results are affects by additional classes.
In particular, we observe that regular PGD becomes less efficient as the number of confusing classes decreases (which is the hallmark of good classifiers).
This analysis demonstrates that Figure~\ref{fig:examples} generalizes to any linear model, multiple input dimensions and any convex adversarial input set (in the sense that $\mathcal{S}(x)$ is a convex set), and leads to the following theorem.

\begin{thm}\label{thm:global_linear}
Given a globally linear model $f_\theta$ with $C$ output logits, for any input $x$, \multitargeted is stronger than regular PGD attacks that use the margin loss (or cross-entropy loss) when the number of restart is greater or equal to $C - 1$ (i.e., $N_\textrm{r} \geq C - 1$) and the adversarial input set $\mathcal{S}(x)$ is convex.
In fact, \multitargeted always finds an optimal attack $\xi^\star$ in this case (i.e., $\nexists \xi \in \mathcal{S}(x)$ s.t. $\hat{L}(f_\theta(\xi, y) > \hat{L}(f_\theta(\xi^\star, y)$).
\end{thm}
\begin{proof}
Since the margin and cross-entropy surrogate losses $\hat{L}(W\xi + b, y)$ used by regular PGD-based attacks are convex with respect to $\xi$, finding their global maximum using gradient ascent depends solely on the initial guess $\xi^{(0)}$.
If $\xi^{(0)}$ is randomly sampled, a regular PGD-based attack has no guarantee that it will find the global maximum as demonstrated in Figure~\ref{fig:examples} (increasing the number of restarts increases the probability that the global maximum is found).
On the other hand, \multitargeted, which inspects all logit differences, is guaranteed to find the global maximum (since the global maximum is contained within one the hyper-planes defined by one of the logit differences).
In the same manner, we can also prove that \multitargeted always finds the optimal attack for any model $f_\theta$ where logit differences are strictly monotonic.
\end{proof}

\subsubsection{Deeper models}
For deeper models, and since the last layer is typically linear, we can reason about the efficiency of \multitargeted by analyzing the shape of the propagated adversarial input set $\mathcal{Z}(x) = \{ f_\theta(\xi) | \xi \in \mathcal{S}(x) \}$.
If $\mathcal{Z}(x)$ is convex, we fallback onto the analysis made in the previous paragraph, for which we know that \multitargeted is more effective than regular PGD.
When $\mathcal{Z}(x)$ is non-convex, the efficiency of \multitargeted cannot be proven.
In fact, we can artificially build examples for which regular PGD is better, and vice-versa (see Appendix~\ref{sec:nonconvex_examples}).
However, we intuitively expect \multitargeted to do well when the behavior $f_\theta$ is close to linear within $\mathcal{S}(x)$, since convexity is preserved under affine transformations.
Previous work~\citep{qin2019adversarial,moosavi-dezfooli_robustness_2018} observed that robust classifiers (at least those trained via adversarial training, \citealp{madry_towards_2017}) tend to behave linearly in the neighborhood of training examples.
Our experimental results confirm these findings: for complex datasets beyond \mnist, for which robust models need to enforce smoothness, \multitargeted is a stronger attack than regular PGD (Appendix~\ref{sec:linearity} gives more details).
Finally, we point out that this analysis generalizes to any specification beyond robustness to \linf-bounded perturbations.
We can summarize our findings in the following theorem.

\begin{figure*}[t]
\centering
  \begin{minipage}[c]{0.45\textwidth}
    \includegraphics[width=\linewidth]{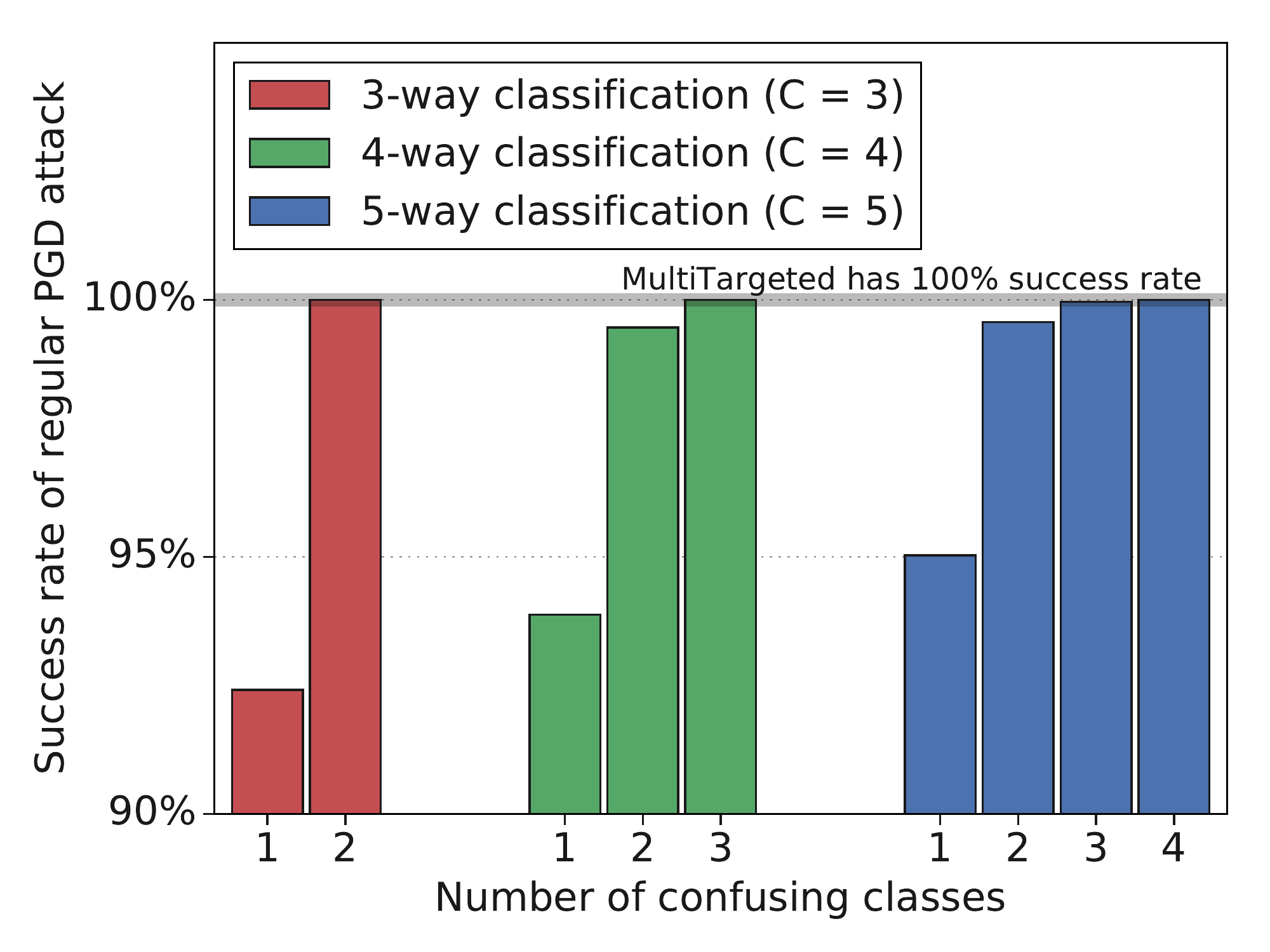}
  \end{minipage}\hfill
  \begin{minipage}[c]{0.45\textwidth}
    \caption{Success rate of a regular PGD-based attack with $C-1$ restarts against a random linear classifier depending on the number of confusing classes. For all possible numbers of confusing classes, \multitargeted has a 100\% success rate. \label{fig:numerical_example}}
  \end{minipage}
\end{figure*}

\begin{thm}\label{thm:local_linear}
For any input $x$ and a convex adversarial input set $\mathcal{S}(x)$, if $f_\theta$ has $C$ outputs and $f_\theta$ is locally linear around $x$ (i.e., $\exists W, b$ s.t. $\forall \xi in \mathcal{S}(x), f_\theta(x) = Wx + b$), \multitargeted always finds an optimal attack $\xi^\star$ within $N_\textrm{r} = C - 1$ restarts (i.e., $\nexists \xi \in \mathcal{S}(x)$ s.t. $\hat{L}(f_\theta(\xi, y) > \hat{L}(f_\theta(\xi^\star, y)$).
\end{thm}
\begin{proof}
The proof follows directly from Theorem~\ref{thm:global_linear} and the fact that, under a linear transformation, the propagated adversarial input set $\mathcal{Z}(x) = \{ f_\theta(\xi) | \xi \in \mathcal{S}(x) \}$ remains convex.
\end{proof}

\section{Experimental Analysis}\label{sec:results}

In this section, we run \multitargeted on widely available models and compare it against regular well-tuned untargeted PGD-based attacks.
To avoid confusion, we denote by PGD$^{K_1 \times N_\textrm{r}}$ the regular PGD-based attack with $N_\textrm{r}$ restarts and $K_1$ PGD steps that reaches the highest success rate among the different variants explained in Section~\ref{sec:preliminaries} (that excludes \multitargeted attacks, but includes FGSM$^K$).
We also denote by MT$^{K_2 \times N_\textrm{i} \times T}$ the \multitargeted attack that targets the top-$T$ classes, uses $N_\textrm{i}$ restarts per target and does $K_2$ PGD steps.
When we omit $T$, we target all classes other than the true class.
To keep the number of compute equivalent, we compare attacks such that $K_1 \times N_\textrm{r} = K_2 \times N_\textrm{i} \times T$ (when $T$ is omitted, for \mnist and \cifar, it is equal to 9).
Unlike of regular PGD$^{K_1 \times N_\textrm{r}}$ (which we tune using grid-search) -- and because we did not find \multitargeted to be as sensitive to hyperparameters as other PGD-based attacks -- we fix MT$^{K_2 \times N_\textrm{i} \times T}$ to use Adam as its optimizer, set the initial step size at 0.1 and reduce it to 0.01 and 0.001 after $K_2 / 2$ and $3 K_2 / 4$ steps, respectively (this schedule is also included in the grid-search performed on regular PGD).

\paragraph{\cifar.}

We focus our analysis on four models.
All models are \wideresnet variants.
The first is adversarially trained by~\citeauthor{madry_towards_2017} (trained against $\epsilon = 8/255$) and is available from MadryLab's white-box \cifar leaderboard.\footnote{\url{https://github.com/MadryLab/cifar10_challenge}; the model used here is the \texttt{adv\_trained} model.}
The second model is trained using TRADES~\citep{zhang2019theoretically} and is available on the TRADES leaderboard.\footnote{\url{https://github.com/yaodongyu/TRADES}} It is empirically more robust than the first network.
The third model is trained using Unsupervised Adversarial Training (UAT, \citealp{uesato2019labels}) and is empirically more robust than the other two models but is three times deeper (with 106 layers).
The last model is adversarially trained by us using the learning rate schedule proposed by \cite{zhang2019theoretically}.
This model is as robust as the \citeauthor{zhang2019theoretically}'s model under regular PGD-based attacks.

\begin{figure*}[t]
\centering
\begin{subfigure}{0.49\textwidth}
\centering
\includegraphics[width=\linewidth]{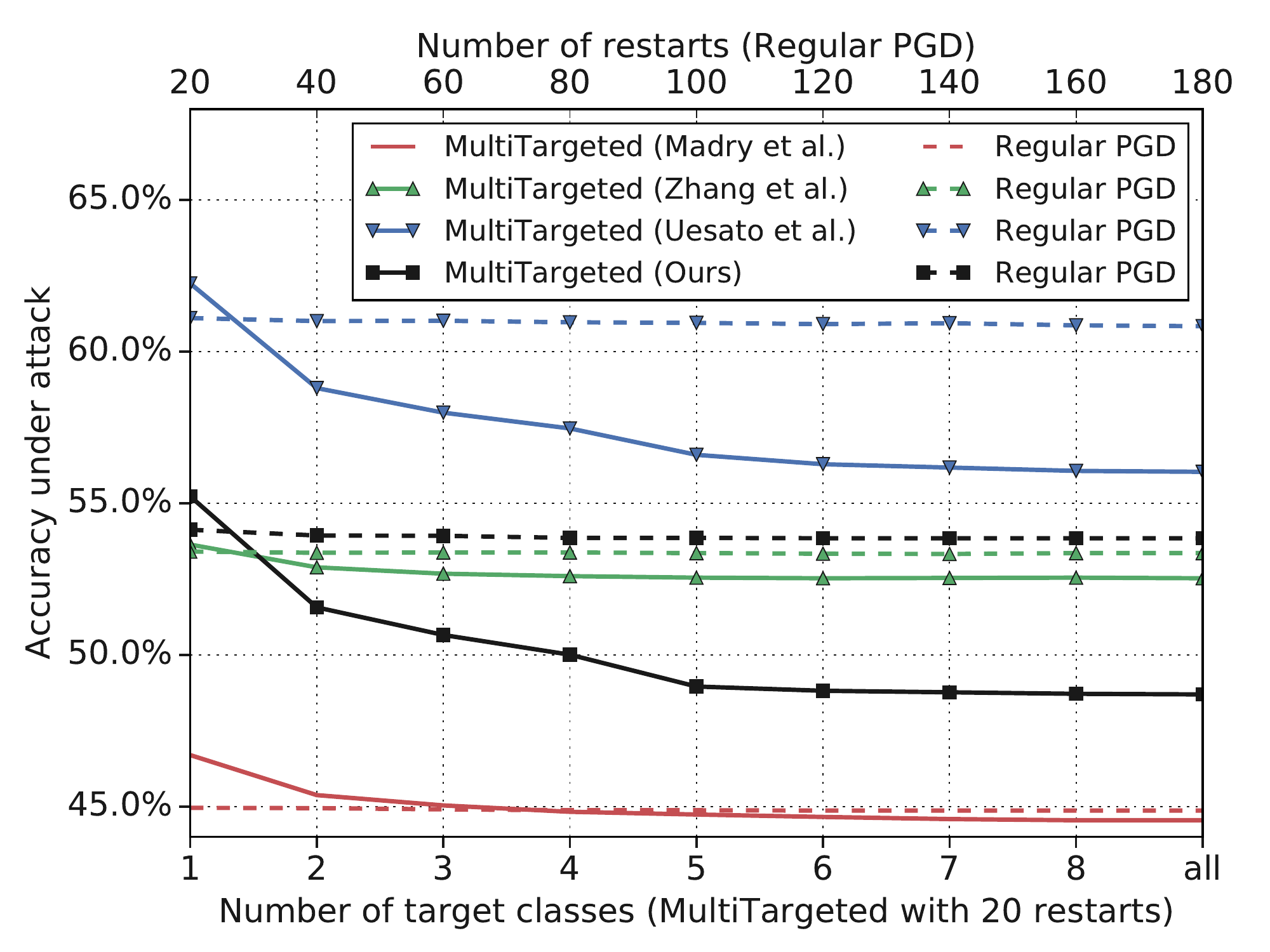}
\caption{\label{fig:cifar_combined2}}
\end{subfigure}
\begin{subfigure}{0.49\textwidth}
\centering
\includegraphics[width=\linewidth]{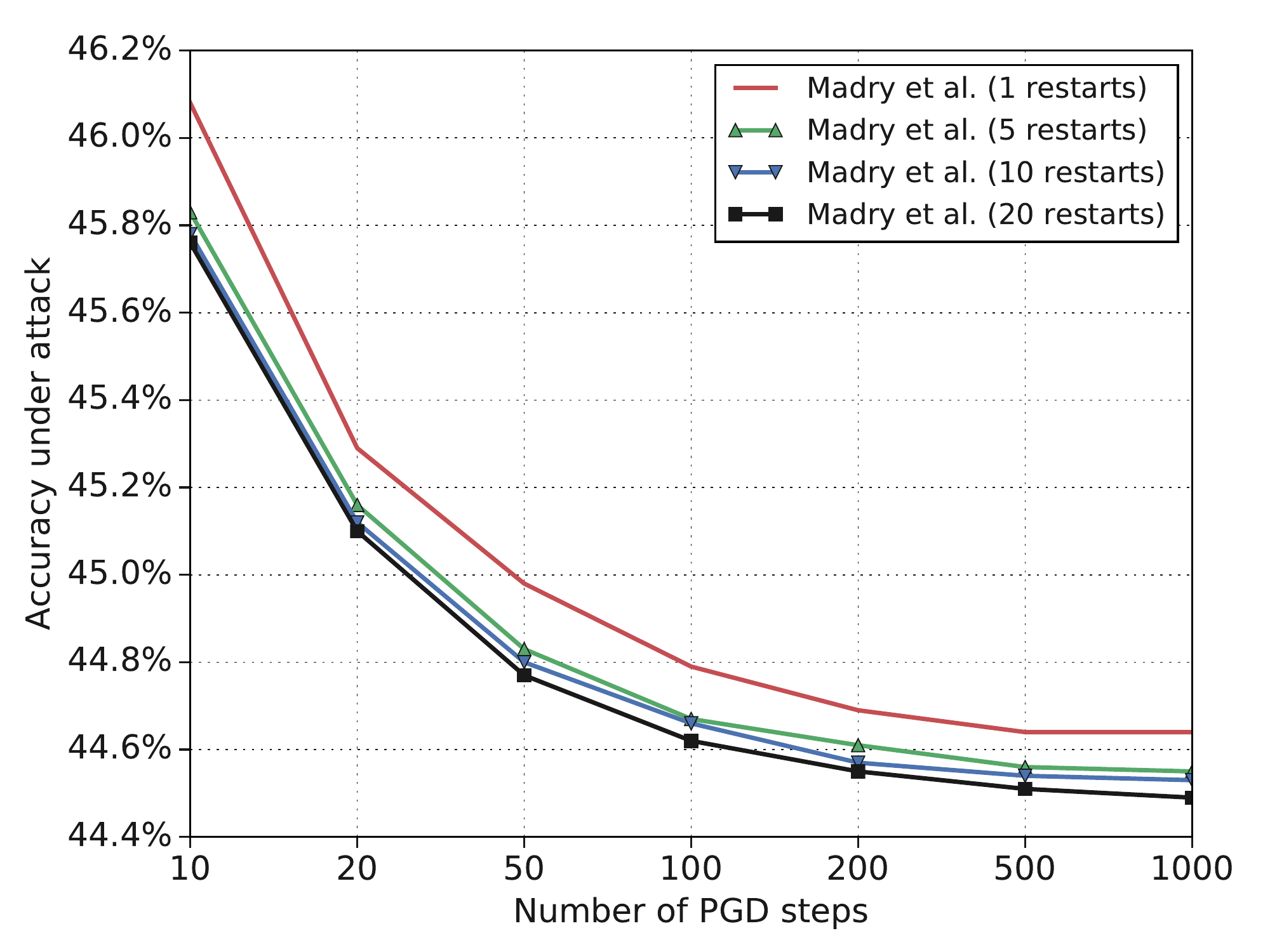}
\caption{\label{fig:cifar_restarts}}
\end{subfigure}
\caption{Panel \subref{fig:cifar_combined2} shows the accuracy under attacks of size $\epsilon = 8/255$ for four different \cifar models. The solid lines are the accuracies obtained by \multitargeted for different values of $T$, while the dashed lines of the corresponding color are accuracies obtained by regular PGD (for the same computational budget).
Panel \subref{fig:cifar_restarts} shows the accuracy under a \multitargeted attack of size $\epsilon = 8/255$ for \citeauthor{madry_towards_2017}'s model as a function of the number of PGD steps $K$ and the number of restarts $N_\textrm{i}$.}
\label{fig:cifar}
\end{figure*}

Figure~\ref{fig:cifar_combined2} shows the accuracy reached by \multitargeted as a function of the number of target classes $T$ at $\epsilon = 8/255$. It compares MT$^{200\times 20 \times T}$ with PGD$^{200 \times (20 \cdot T)}$.
For all models, \multitargeted is the strongest attack.
In fact, even when targeting only the top-2 classes, \multitargeted significantly outperforms regular PGD on all models (except \citeauthor{madry_towards_2017}'s model).
We also observe that the performance of regular PGD stagnates and that adding more restarts is ineffective, which appears to confirm that for some examples in the test set there is a low probability of finding an initial guess $\xi^{(0)}$ that leads to misclassified perturbation (as examplified in Figure~\ref{fig:motivating_example2}).
For completeness, Figure~\ref{fig:cifar_restarts} shows how the accuracy under an MT$^{K \times N_\textrm{i}}$ attack evolves as a function of the number of PGD steps $K$ and the number of restarts $N_\textrm{i}$. We note that, while this figure shows specifically what happens on \citeauthor{madry_towards_2017}'s model, the behavior is identical on the other three \cifar models. As expected, the strength of the attack saturates with more restarts and steps. We also observe that \multitargeted with a single restart is more effective than a regular PGD-based attack with hundreds of restarts.

Finally, Table~\ref{table:leaderboard}, at the end of this section, summarizes the performance of each attack on models that have an online leaderboard (i.e., \citeauthor{madry_towards_2017}'s and \citeauthor{zhang2019theoretically}'s models).
For both models, \multitargeted ranks first.

\paragraph{\imagenet.}

We focus our analysis on three models.
All models are adversarially trained against random targeted attacks at $\epsilon = 16/255$ (as is usual on \imagenet).
However, they all exhibit non-trivial empirical robustness to untargeted attacks; and we evaluate them in that context (as \multitargeted is an untargeted attack).
The first model is a standard \resnet-152: it is available from \citeauthor{xie_feature_2018}'s GitHub page.\footnote{\url{https://github.com/facebookresearch/ImageNet-Adversarial-Training}}
The second model is a variant of \resnet-152 that uses additional ``denoise'' blocks: it is also trained by \cite{xie_feature_2018}.
The third model is trained by ourselves: we put emphasis on robustness under attack rather than accuracy on clean examples.
As a result (on the first 1000 images of the test set) the last model has a clean accuracy of 51.4\% compared to 66.8\% for the ``denoise'' model and 64.1\% for the regular model by \citeauthor{xie_feature_2018}.

Figure~\ref{fig:imagenet_combined} shows the accuracy reached by \multitargeted as a function of the number of target classes $T$ at $\epsilon = 16/255$. It compare MT$^{200\times 20 \times T}$ with PGD$^{200 \times (20 \cdot T)}$.
For all models \multitargeted is a better attack (yielding lower accuracy) starting with $T = 2$.
Another interesting observation is the fact that under the \multitargeted attack, the standard and ``denoise'' variants trained by \citeauthor{xie_feature_2018} are on par (the regular PGD-based attack would indicate otherwise).\footnote{This analysis is performed using an untargeted attack and these conclusions are not necessarily at odds with the results shown by~\cite{xie_feature_2018}, which looks at random targeted attacks.}
This confirms the observation made in Section~\ref{sec:example}.
Indeed, \cite{xie_feature_2018} demonstrate that under random targeted attacks the ``denoise'' model is stronger.
As such, there are less confusing classes, which can explain why \multitargeted is performing well in comparison to regular PGD on the ``denoise'' model.

\begin{figure*}[t]
\centering
  \begin{minipage}[c]{0.49\textwidth}
    \includegraphics[width=\linewidth]{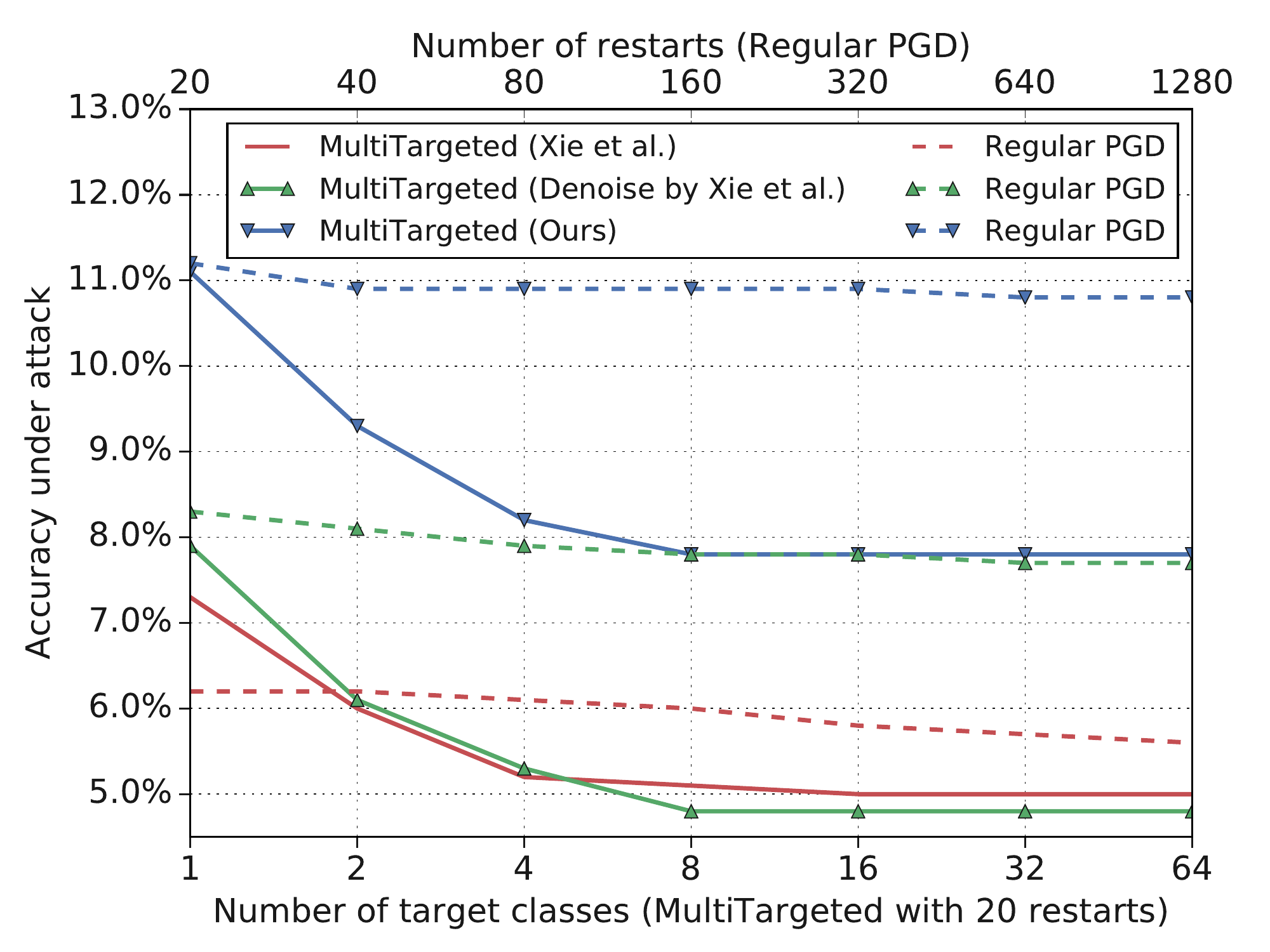}
  \end{minipage}\hfill
  \begin{minipage}[c]{0.45\textwidth}
    \caption{Accuracy under attacks of size $\epsilon = 16 / 255$ for three different \imagenet models. The solid lines are the accuracies obtained by \multitargeted for different values of $T$, while the dashed lines of the corresponding color are accuracies obtained by regular PGD (for the same computational budget). \label{fig:imagenet_combined}}
  \end{minipage}
\end{figure*}

\paragraph{Leaderboards on \mnist and \cifar.}

In this section, we introduce a combined PGD+MT$^{K \times N_\textrm{i}}$ attack that alternates between the logit differences shown in Equation~\eqref{eq:multitargeted_loss} and the margin loss in Equation~\eqref{eq:margin_loss}.
Thus, for \mnist and \cifar, where the number of classes is ten, the total number of surrogate losses is ten and thusthe total number of gradient evaluation is $K \times N_\textrm{i} \times 10$ (which results in using about 11\% more compute that \multitargeted).
Algorithm~\ref{alg:pgd_mt} in the supplementary material shows an implementation of this attack.

For \mnist, we analyze the model adversarially trained by~\citeauthor{madry_towards_2017} (trained against $\epsilon = 0.3$). It is available from MadryLab's white-box \mnist leaderboard.\footnote{\url{https://github.com/MadryLab/mnist_challenge}; the model used here is the \texttt{secret} model.}
On the leaderboard, it states that this model achieves 89.62\% accuracy under PGD$^{100\times50}$.
With proper tuning, our implementation of PGD$^{100\times50}$ reaches 89.03\% accuracy, which is closer to the state-of-art results obtained by~\cite{wang_enhancing_2019} with 88.42\%.
For \cifar, we re-analyze two \wideresnet models.
The first model from \citeauthor{madry_towards_2017} is available on MadryLab's leaderboard,\footnote{\url{https://github.com/MadryLab/cifar10_challenge}; the model used here is the \texttt{secret} model.} where it achieves 45.21\% accuracy under PGD$^{20\times10}$ (at $\epsilon = 8 / 255$).
Our tuned PGD$^{20\times10}$ variant reaches a similar accuracy of 45.18\% and the state-of-the-art is obtained by~\cite{croce2019minimally} with 44.51\%.
The second model of \citeauthor{zhang2019theoretically} is from the TRADES leaderboard, where it states that FGSM$^{1000}$ reaches 56.53\% in accuracy (at $\epsilon = 0.031)$.
Our tuned PGD$^{50\times20}$ variant reaches 54.05\% accuracy.
The state-of-the-art obtained by~\citeauthor{croce2019minimally} is 53.44\%.

Table~\ref{table:leaderboard} summarizes the results.
Overall, this table illustrates well the dichotomy and complementary of regular PGD and \multitargeted (see Section~\ref{sec:method}).
On \mnist, for the same number of gradient evaluations, regular PGD tends to do better than \multitargeted~-- while on \cifar, it is the opposite.
These results do seem to indicate that, for \citeauthor{madry_towards_2017}'s \mnist model, the number of restarts is much more important than the choice of surrogate loss.
As explained in Appendix~\ref{sec:linearity}, this model does not behave linearly and is not smooth in the neighborhood of datapoints.
As such, the propagated adversarial input set $\mathcal{Z}(x)$ can be highly non-convex and, thus, more restarts are required to avoid local minima.
In fact, combining both regular PGD and \multitargeted into a PGD+MT attack does not provide a significant advantage (beyond adding a few additional restarts).
On both \mnist and \cifar, MT itself achieves state-of-the-art numbers.
For both datasets, the best attack misses only a handful of adversarial examples (as visible by the small gap in accuracy with the ``aggregated'' column, which aggregates results across the three attack variants).

\begin{table}[t]
\begin{center}
\footnotesize{
\begin{tabular}{lc|cc|ccc}
    \hline
    \cellcolor{Highlight} & \cellcolor{Highlight} & \multicolumn{5}{c}{\cellcolor{Highlight} \bf Accuracy under attack} \\
    
    \cellcolor{Highlight} \mnist & \cellcolor{Highlight} {\bf $\epsilon$} & \cellcolor{Highlight} \textbf{PGD$^{1000\times1800}$} & \cellcolor{Highlight} \textbf{MT$^{1000\times200}$} & \cellcolor{Highlight} \textbf{\textbf{PGD+MT$^{1000\times200}$}} & \cellcolor{Highlight} \textrm{Aggregated} & \cellcolor{Highlight} \textrm{\citeauthor{wang_enhancing_2019}} \\
    \hline
    \citeauthor{madry_towards_2017} & 0.3 & \textbf{88.21\%} & 88.43\% & 88.36\% & 88.18\% & 88.42\% \TBstrut \\
    \hline
    \multicolumn{7}{c}{\vspace{-2mm}} \\
    \hline
    \cellcolor{Highlight} \cifar & \cellcolor{Highlight} {\bf $\epsilon$} & \cellcolor{Highlight} \textbf{PGD$^{1000\times180}$} & \cellcolor{Highlight} \textbf{MT$^{1000\times20}$} & \cellcolor{Highlight} \textbf{\textbf{PGD+MT$^{1000\times20}$}} & \cellcolor{Highlight} \textrm{Aggregated} & \cellcolor{Highlight} \textrm{\citeauthor{croce2019minimally}} \Tstrut \\
    \hline
    \citeauthor{madry_towards_2017} & 8/255 & 44.51\% & \textbf{44.05\%} & 44.03\% & 44.03\% & 44.51\% \Tstrut \\
    \citeauthor{zhang2019theoretically} & 0.031 & 53.70\% & \textbf{53.07\%} & 53.07\% & 53.05\% & 53.44\% \Bstrut \\
    \hline
\end{tabular}
}
\end{center}
\caption{Comparison with state-of-the-art attacks. On \mnist (top-half of the table), the best known attack is the ``Interval Attack'' by \cite{wang_enhancing_2019}. On \cifar (bottom-half of the table), the best known attack is the ``Fast Adaptive Boundary Attack'' by \cite{croce2019minimally}. For both dataset, we compare regular PGD with \multitargeted.
The penultimate column is the accuracy obtained when combining all attacks into one (i.e., picking the worst-case adversarial example from any of the attacks).
\label{table:leaderboard}}
\end{table}

\section{Conclusion}

This paper provides an overview of projected gradient descent and its use for rigorous testing.
We introduce the concept of \multitargeted attacks, which is a variant of PGD, and demonstrate that we can reach state-of-art results on three separate leaderboards.
We highlight under which conditions \multitargeted is more effective.
We hope that this paper can serve as an inspiration for tuning PGD-based attacks and motivates why we should not necessarily rely on existing hyper-parameter values.

\bibliographystyle{abbrvnat}
\setlength{\bibsep}{5pt} 
\setlength{\bibhang}{0pt}
\bibliography{bibliography}

\clearpage
\appendix
\section{Combining regular PGD and \multitargeted}

As shown in Section~\ref{sec:method}, there exists cases where the margin loss is more efficient than using each logit difference (as proposed by \multitargeted).
Algorithm~\ref{alg:pgd_mt} below combines both approaches into a single PGD+MT$^{K \times N_\textrm{i}}$ attack.

\begin{algorithm}[h]
\caption{PGD+MT attack}
\begin{algorithmic}[1]
\Require A nominal input $x$ and label $y$, a model $f_\theta(x)$ and a threat model $\mathcal{S}(x)$
\Ensure Possible attack $\hat{\xi} \in \mathcal{S}(x)$
\State $\mathcal{T} \gets \{1, \ldots, C\} \setminus \{y\}$  \Comment{$C$ is the number of classes.}
\For{$i \in \{1, \ldots, N_\textrm{i} \}$}
  \For{$j \in \{1, \ldots, C\}$}
    \If{$j = C$}
        \State Use surrogate loss $\hat{L}^{(i C + j + 1)}(z, y) = -z_y + \max_{t \in \mathcal{T}} z_t$ \Comment{margin loss}
    \Else
        \State Use surrogate loss $\hat{L}^{(i C + j + 1)}(z, y) = -z_y + z_t$ with $t = \mathcal{T}_j$ \Comment{logit difference losses}
    \EndIf
    \State Initialize optimizer \texttt{Opt}
  \State $\xi^{(0)} \gets \texttt{SampleFrom}(x, \mathcal{S})$
  \For{$k \in \{1, \ldots, K\}$}
    \State $\xi^{(k)} \gets \proj_{\mathcal{S}(x)} \left( \xi^{(k-1)} + \alpha_k \texttt{Opt} \left( \grad_{\xi^{(k-1)}} \hat{L}^{(r)}(f_\theta(\xi^{(k-1)}), y) \right) \right)$
    \If{$L(f_\theta(\xi^{(k)}), y) > L(f_\theta(\hat{\xi}, y)$}
      \State $\hat{\xi} \gets \xi^{(k)}$
    \EndIf
  \EndFor
  \EndFor
\EndFor
\end{algorithmic}
\label{alg:pgd_mt}
\end{algorithm}

\clearpage
\section{Non-convex adversarial input sets}\label{sec:nonconvex_examples}

This section shows why, under non-convex adversarial input sets, \multitargeted is not necessarily more efficient than regular PGD-based attacks -- even for linear models (for the same compute budget).
Let us consider a 2 dimensional input $x$ and a 3-way classification task.
Figure~\ref{fig:nonconvex_example} shows the margin loss for two slightly different settings.
For all panels, the adversarial input set is highlighted in blue and the region in purple is where misclassification occurs.
When $\xi^{(0)}$ is sampled within the region highlighted in red, the attack for which that panel corresponds to is successful.

At the top, in Figure~\ref{fig:nonconvex_example1_mt} and~\subref{fig:nonconvex_example1_pgd}, we illustrate an example for which regular PGD using the margin loss is more likely to find the worst-case attack (located in the top-right corner).
Indeed, the region in red is larger in panel~\subref{fig:nonconvex_example1_pgd} than in panel~\subref{fig:nonconvex_example1_mt}.
On the other hand, in Figure~\ref{fig:nonconvex_example2_mt} and~\subref{fig:nonconvex_example2_pgd}, we illustrate the opposite.
A slight shift of the input makes regular PGD less likely to find the optimal perturbation.

\begin{figure*}[b]
\centering
\begin{subfigure}{0.49\textwidth}
\centering
\includegraphics[width=\linewidth]{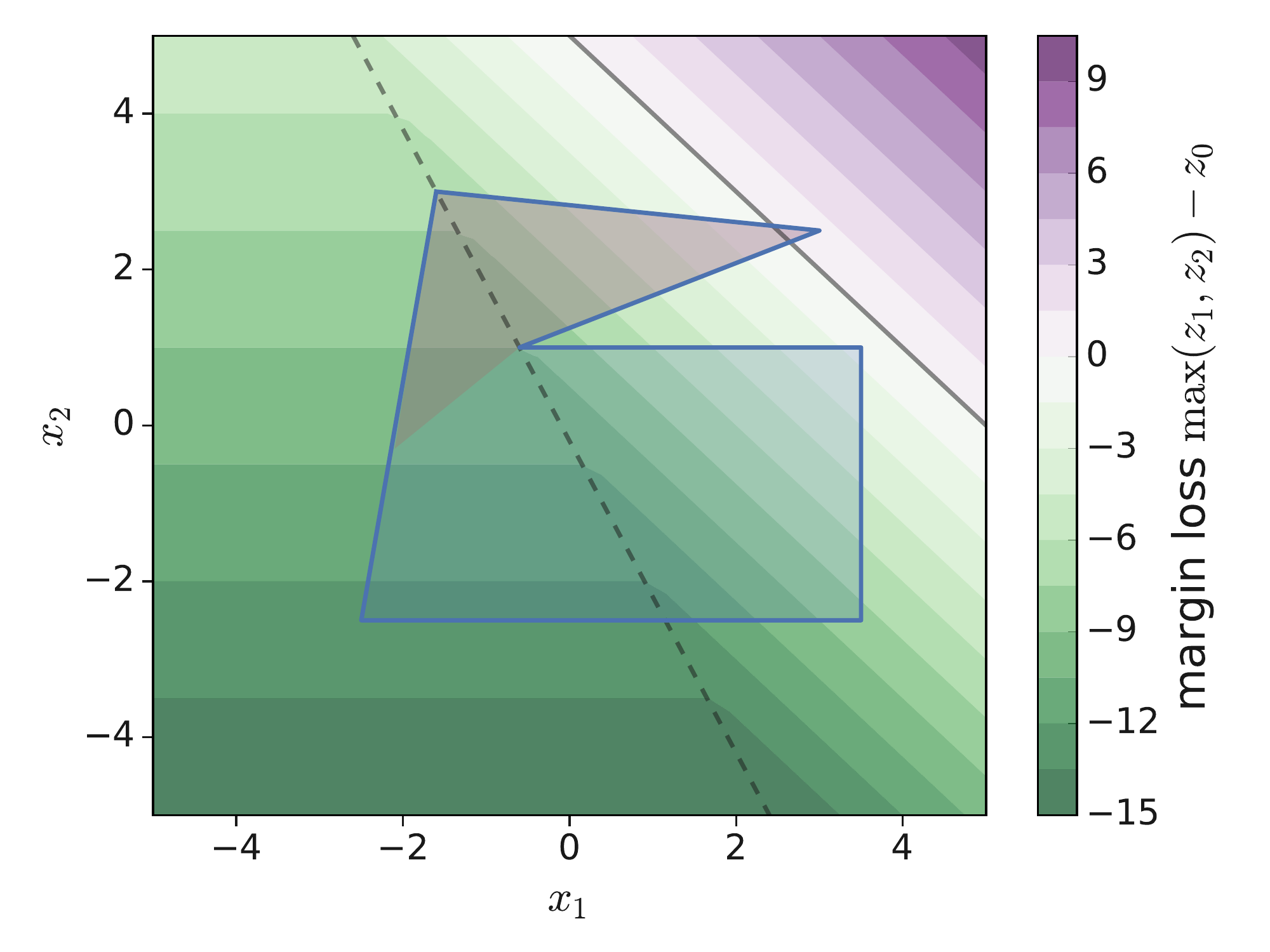}
\caption{Example 1 for \multitargeted \label{fig:nonconvex_example1_mt}}
\end{subfigure}
\begin{subfigure}{0.49\textwidth}
\centering
\includegraphics[width=\linewidth]{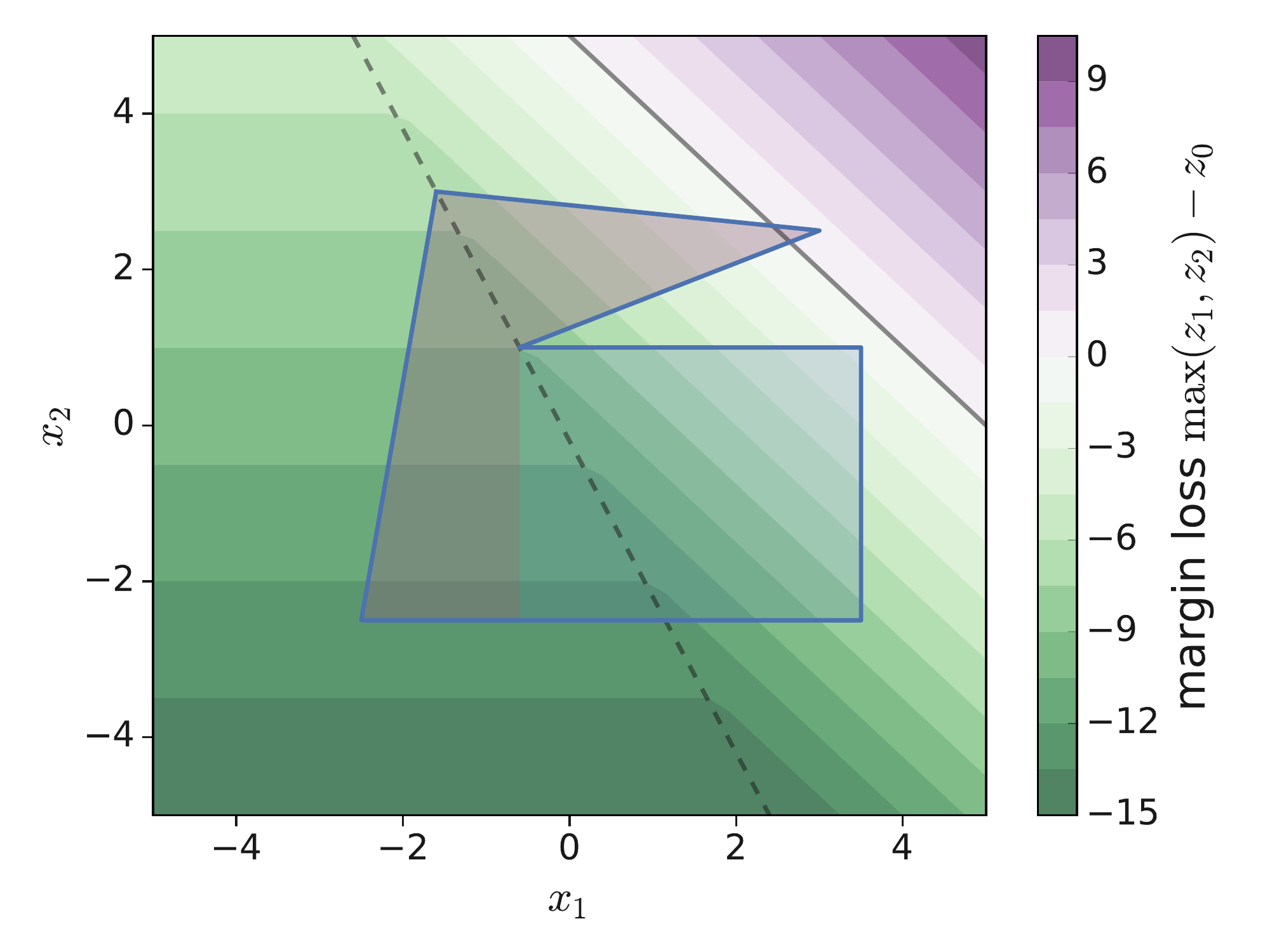}
\caption{Example 1 for regular PGD \label{fig:nonconvex_example1_pgd}}
\end{subfigure}
\begin{subfigure}{0.49\textwidth}
\centering
\includegraphics[width=\linewidth]{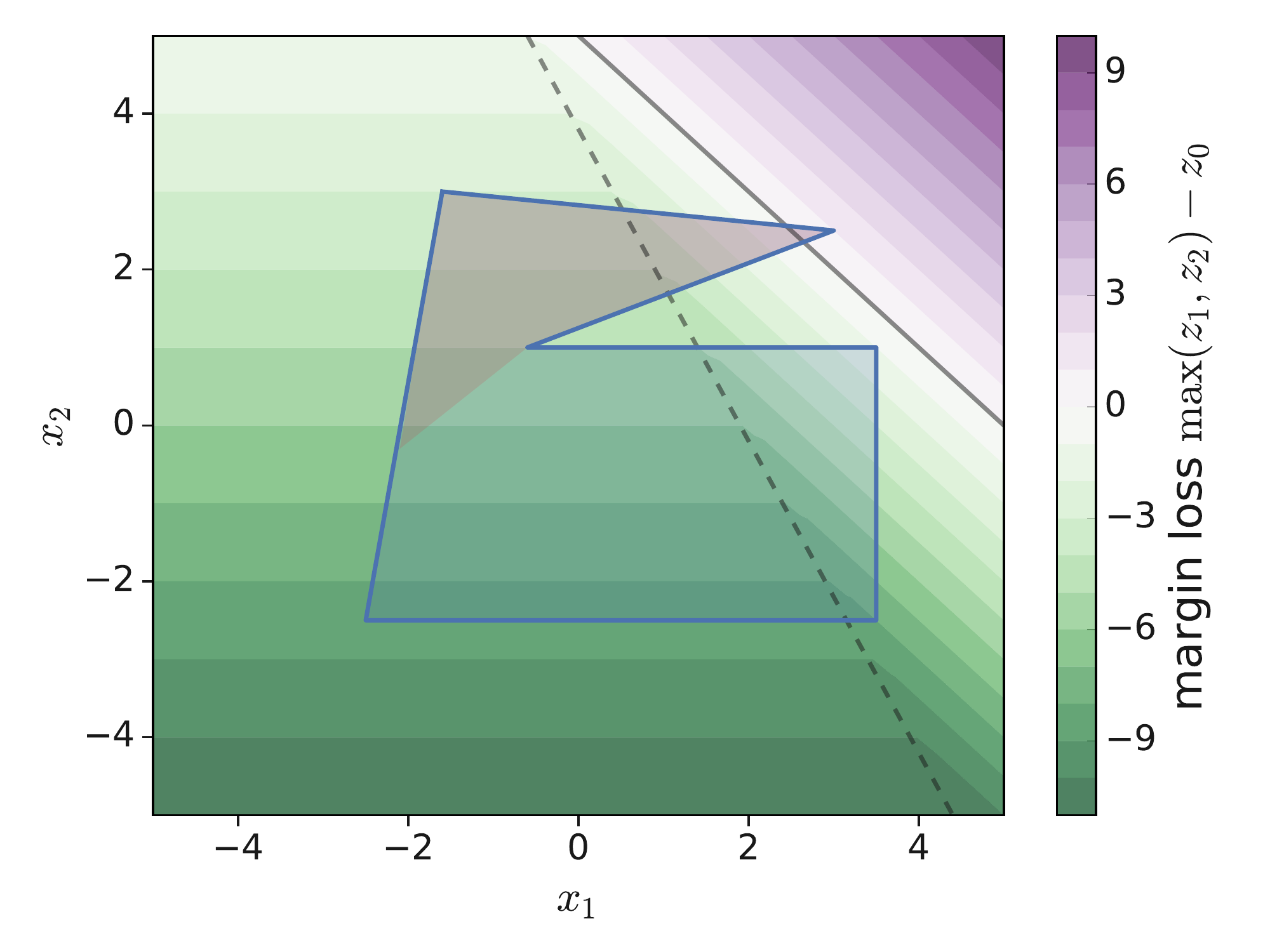}
\caption{Example 2 for \multitargeted \label{fig:nonconvex_example2_mt}}
\end{subfigure}
\begin{subfigure}{0.49\textwidth}
\centering
\includegraphics[width=\linewidth]{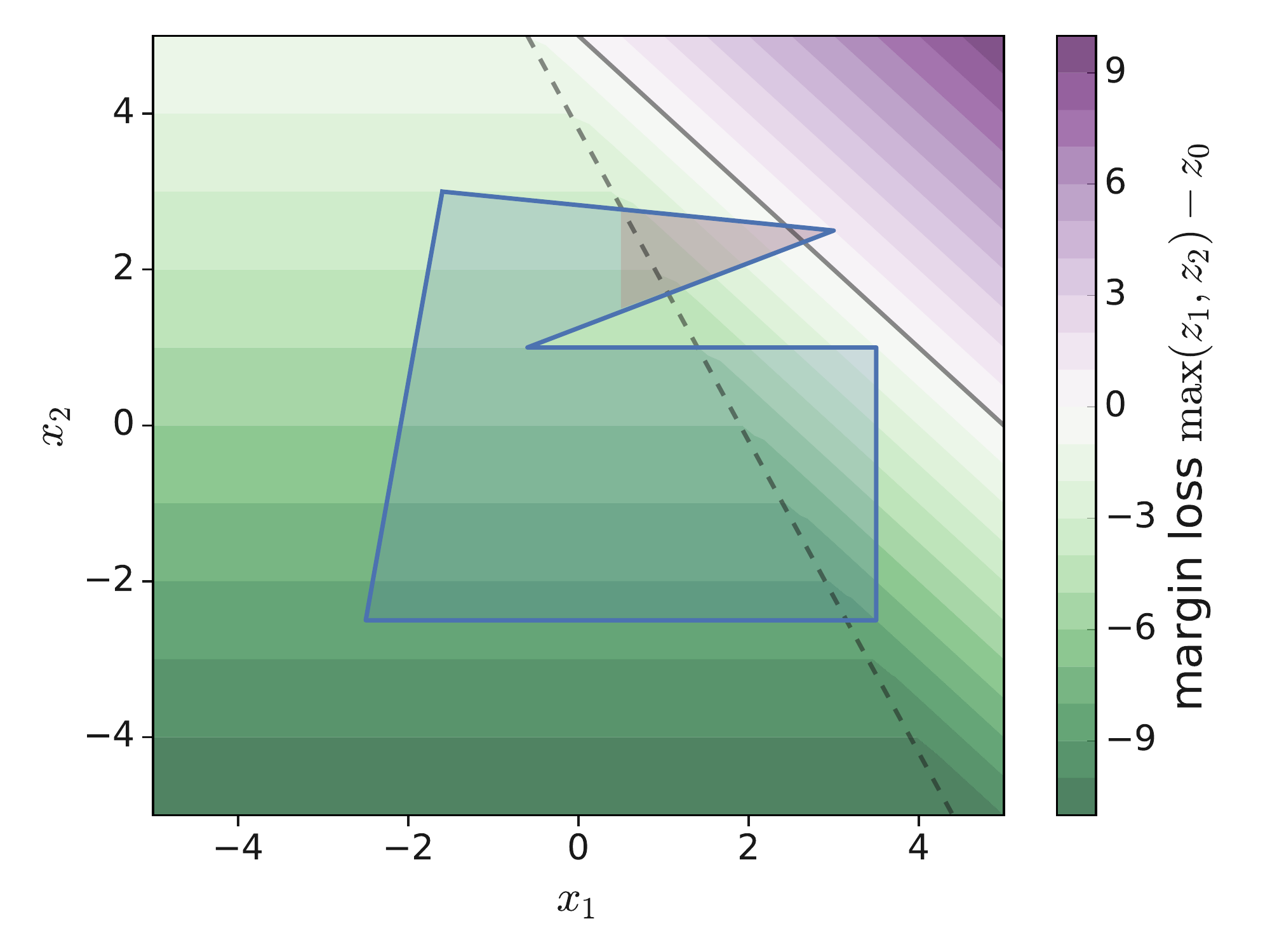}
\caption{Example 2 for regular PGD \label{fig:nonconvex_example2_pgd}}
\end{subfigure}
\caption{Examples with non-convex adversarial input sets.
For all panels, the adversarial input set is highlighted in blue and the region in purple is where misclassification occurs.
The solid line in black defines the boundary where the input is classified as class 2.
The dashed line is where the two plane $z_1 - z_0$ and $z_2 - z_0$ intersect.
When $\xi^{(0)}$ is sampled within the region highlighted in red, the attack for which that panel corresponds to is successful.
At the top, we have an example for which regular PGD is more successful.
At the bottom, we have an example for which \multitargeted is more successful.}
\label{fig:nonconvex_example}
\end{figure*}

\clearpage
\section{Effect of linearity}\label{sec:linearity}

As discussed in Section~\ref{sec:method}, the effectiveness of \multitargeted is dependent on the how the adversarial input set propagates through the network $f_\theta$.
In particular, if the propagated set $\mathcal{Z}(x) = \{ f_\theta(\xi) | \xi \in \mathcal{S}(x) \}$ is convex, \multitargeted is at least as effective as a regular PGD-based attack using the margin or cross-entropy loss (i.e., it always succeeds in finding the optimal attack given enough restarts, see Theorem~\ref{thm:local_linear}).
When the adversarial input set $\mathcal{S}(x)$ is convex, it is important to characterize the local behavior of $f_\theta$ around the nominal input $x$.
Indeed, if $f_\theta$ behaves linearly then $\mathcal{Z}(x)$ is convex.
Our hypothesis is that when the $f_\theta$ is ``close'' to linear \multitargeted is more effective.
To analyze the behavior of a given model, we gather the gradients of all logits with respect to hundred randomly selected inputs around examples from the test set.
When the singular values of the matrix that gathers these gradient are mostly concentrated around a few direction, the model should behave more linearly (in the extreme case where a single singular value is non-zero, the model is linear).

Figure~\ref{fig:singular_values_mnist} shows the singular values of three \mnist models.
The first two models, denoted ``\citeauthor{madry_towards_2017}'' and ``Ours'' are adversarially trained model of the same size (i.e., two convolutional layers followed by two fully connected layers).
The last model, denoted by ``Ours (small)'' is a much smaller model (also adversarially trained).
We observe that the smaller model tends to behave more linearly.
This is also qualitatively shown in Figure~\ref{fig:landscapes_mnist} that shows how the logits change as a function of the input.
Table~\ref{table:mnist_comparison} confirms that under the same compute budget \multitargeted works better on the smaller, more locally linear, model.
For less linear models, it is important to use more restarts and, as such, it is preferential to use a single loss (e.g., margin or cross-entropy loss) and increase the number of restarts.
We highlight that, when given more restarts per target class, \multitargeted can reach the same performance as regular PGD.

\begin{figure*}[b]
\centering
\begin{subfigure}{0.49\textwidth}
\centering
\includegraphics[width=\linewidth]{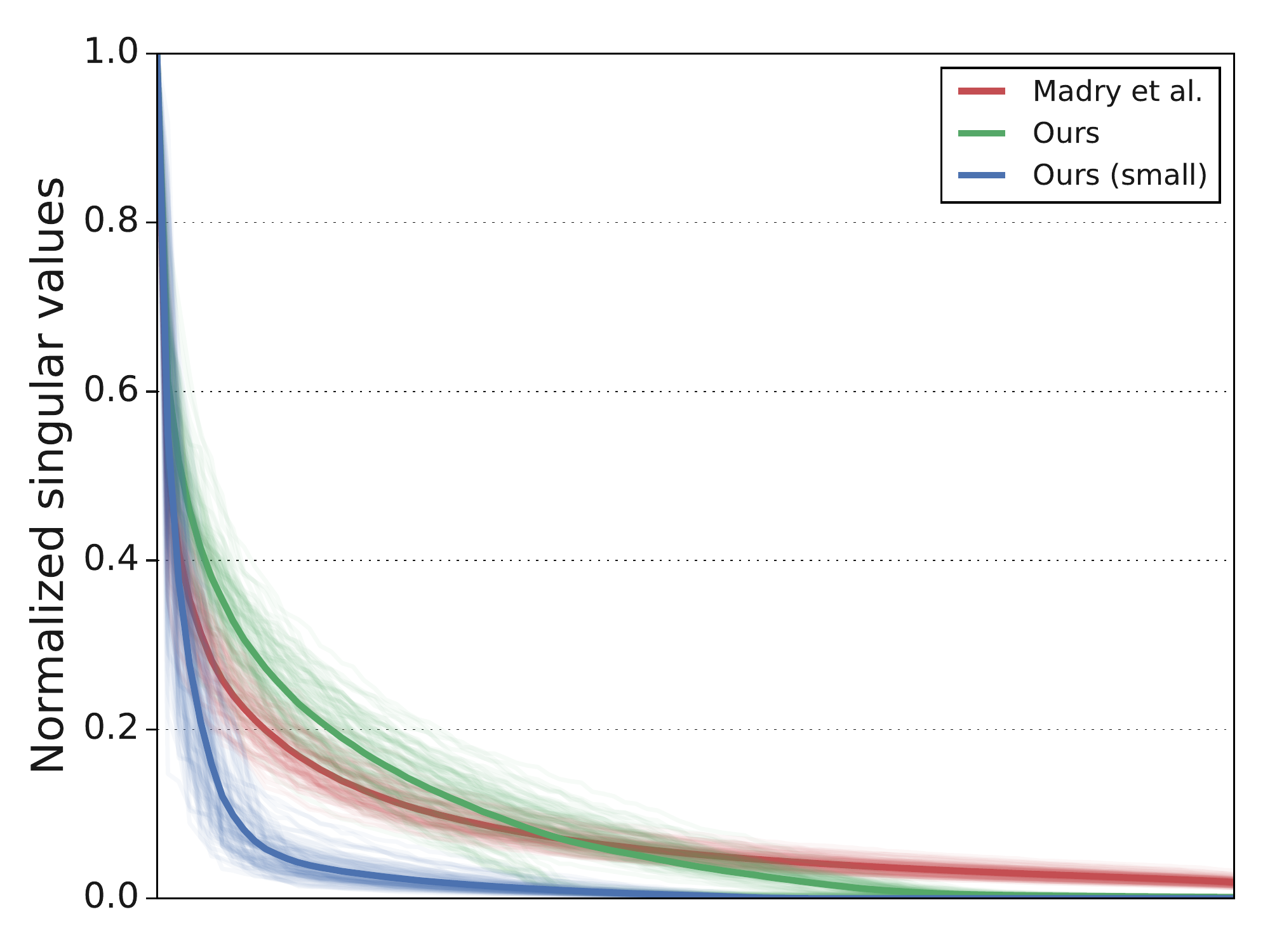}
\caption{\label{fig:singular_values_mnist}}
\end{subfigure}
\begin{subfigure}{0.49\textwidth}
\centering
\includegraphics[width=\linewidth]{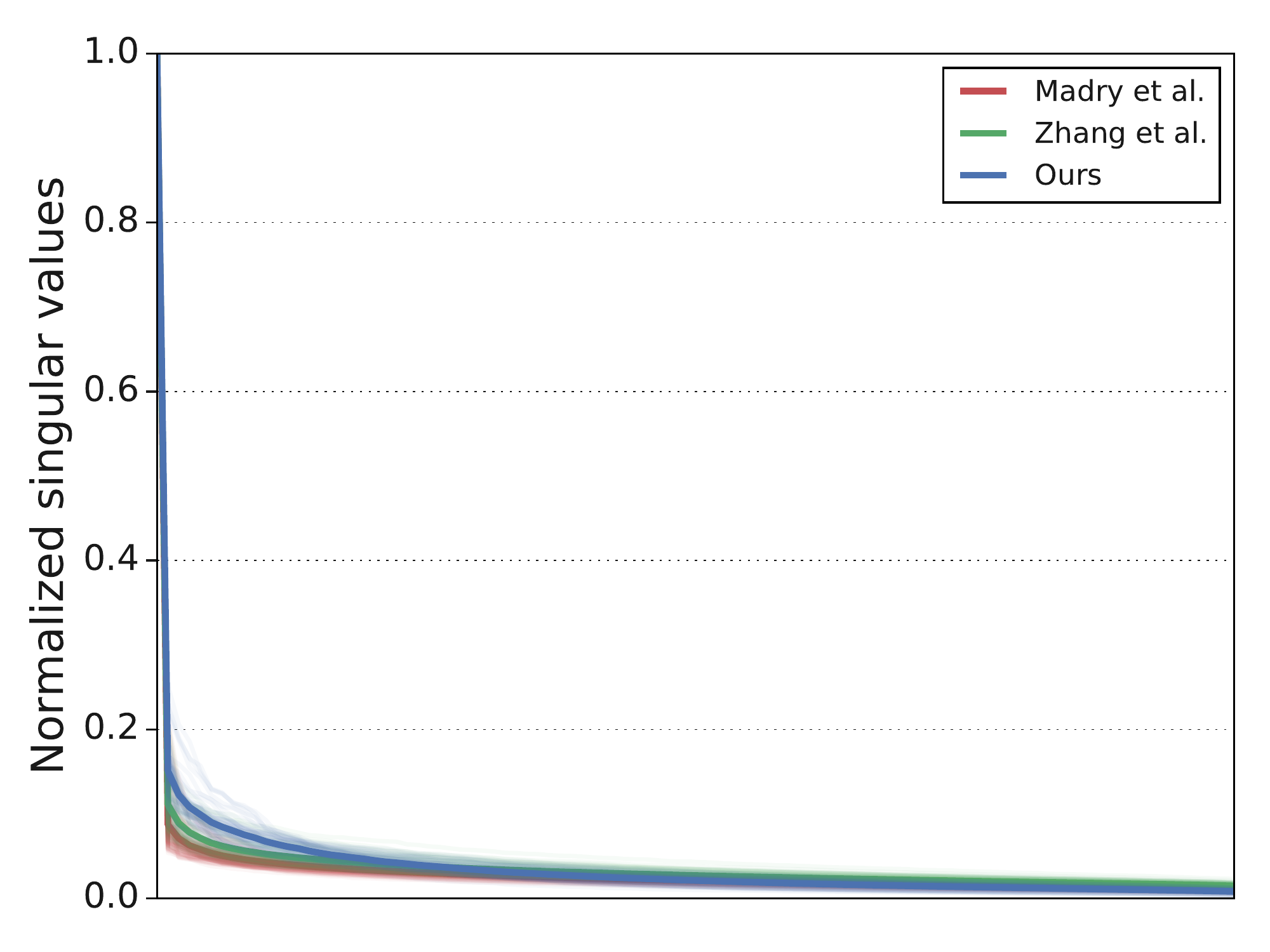}
\caption{\label{fig:singular_values_cifar}}
\end{subfigure}
\caption{Ordered singular values of a matrix gathering the gradients of each logit with respect to hundred randomly selected inputs within the adversarial input set around ten test datapoints (the thicker line represents the average over logits and datapoints). All values are normalized by the largest singular value. As the behavior of the model becomes more linear, we expect the curve to become steeper.
Panel \subref{fig:singular_values_mnist} shows the singular values of three \mnist models.
We observe that the smallest model behaves more linearly than the other two models.
Panel \subref{fig:singular_values_cifar} shows the singular values of three \cifar models.
We observe that \cifar models tend to behave more linearly than \mnist models.}
\label{fig:singular_values}
\end{figure*}

\begin{table}[b]
\begin{center}
\footnotesize{
\begin{tabular}{lc|cc|c}
    \hline
    \cellcolor{Highlight} & \cellcolor{Highlight} & \multicolumn{3}{c}{\cellcolor{Highlight} \bf Accuracy under attack} \\
    
    \cellcolor{Highlight} \mnist & \cellcolor{Highlight} {\bf $\epsilon$} & \cellcolor{Highlight} \textbf{PGD$^{1000\times180}$} & \cellcolor{Highlight} \textbf{MT$^{1000\times20}$} & \cellcolor{Highlight} \textbf{MT$^{1000\times180}$} \\
    \hline
    \citeauthor{madry_towards_2017} & 0.3 & \textbf{88.49\%} & 88.89\% & 88.45\% \\
    Ours & 0.3 & \textbf{91.98\%} & 92.96\% & 92.05\% \\
    Ours (small) & 0.3 & 74.13\% & \textbf{73.56\%} & 73.11\% \\
    \hline
\end{tabular}
}
\end{center}
\caption{Comparison of \multitargeted and regular PGD on three \mnist models. \multitargeted is more effective on the smaller model.
\label{table:mnist_comparison}}
\end{table}

In comparison, Figure~\ref{fig:singular_values_cifar} shows the singular values of three \cifar models.
We noticed in Section~\ref{sec:results} that \multitargeted is significantly better than regular PGD on robustly trained \cifar models.
This plot confirms that these models behave almost linearly (in the local neighborhood of datapoints).
This is also qualitatively shown in Figure~\ref{fig:landscapes_cifar} that shows how the logits change as a function of the input for these models.

\begin{figure*}[b]
\centering
\begin{subfigure}{0.32\textwidth}
\centering
\includegraphics[width=\linewidth]{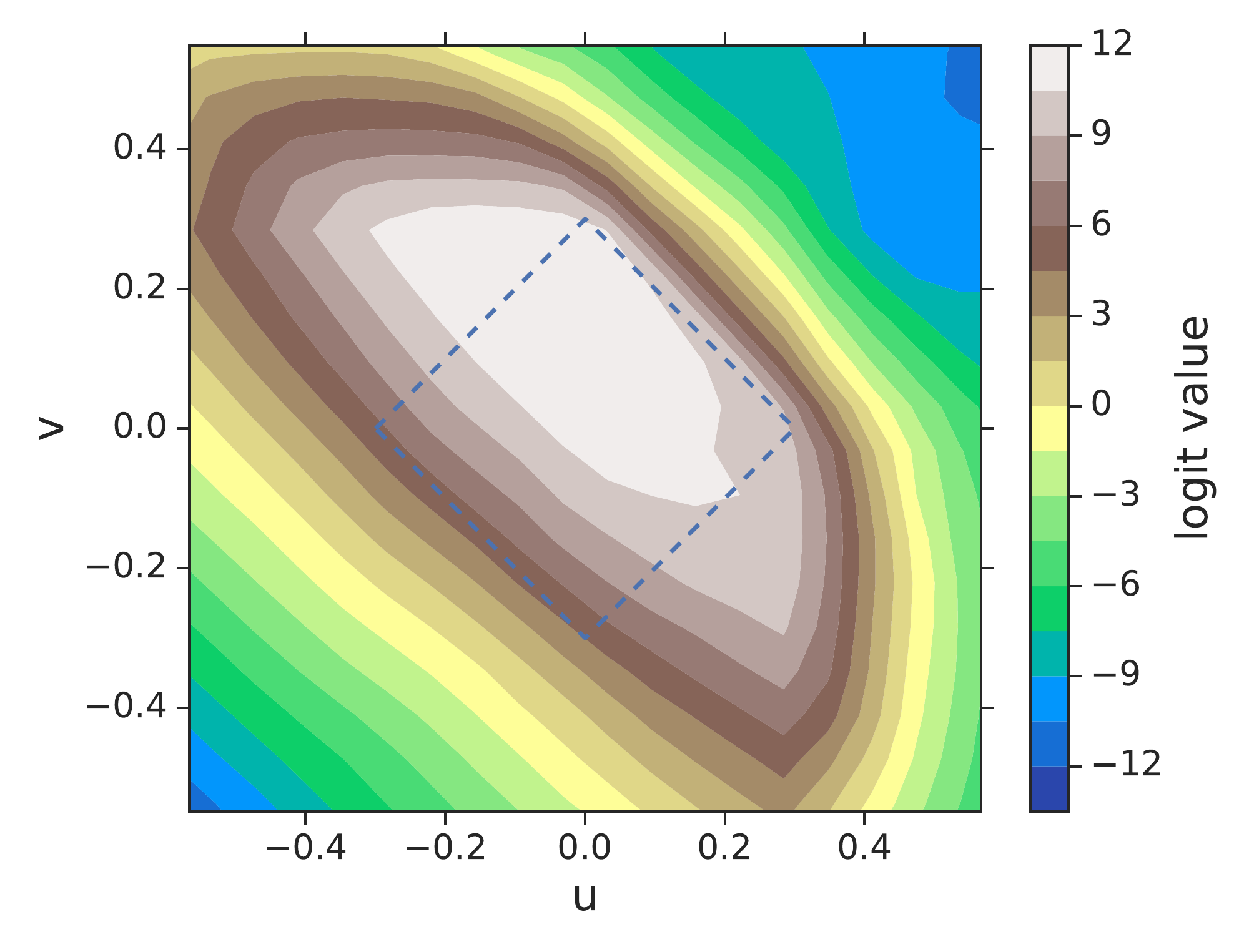}
\caption{$z_7$ on \citeauthor{madry_towards_2017}}
\end{subfigure}
\begin{subfigure}{0.32\textwidth}
\centering
\includegraphics[width=\linewidth]{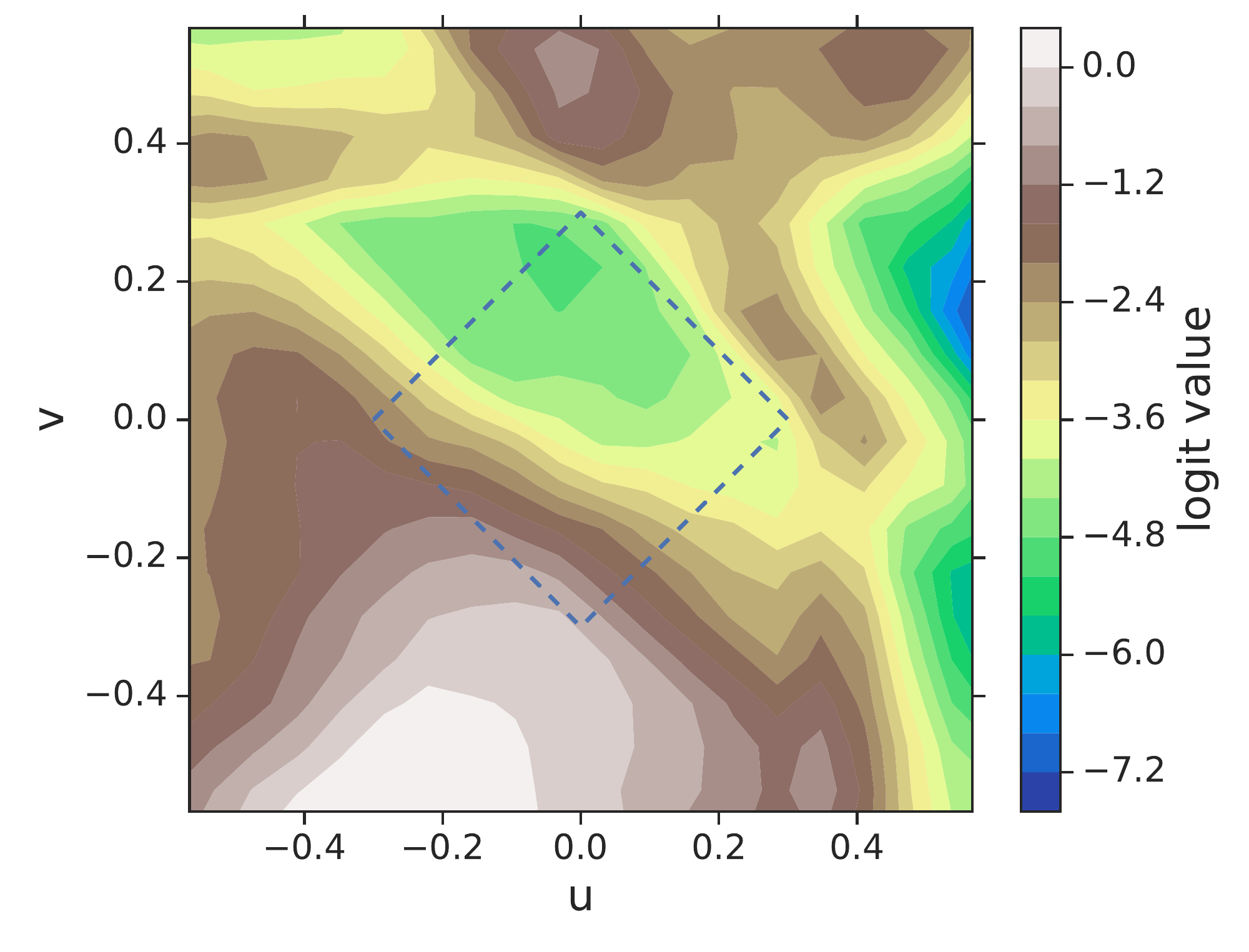}
\caption{$z_0$ on \citeauthor{madry_towards_2017}}
\end{subfigure}
\begin{subfigure}{0.32\textwidth}
\centering
\includegraphics[width=\linewidth]{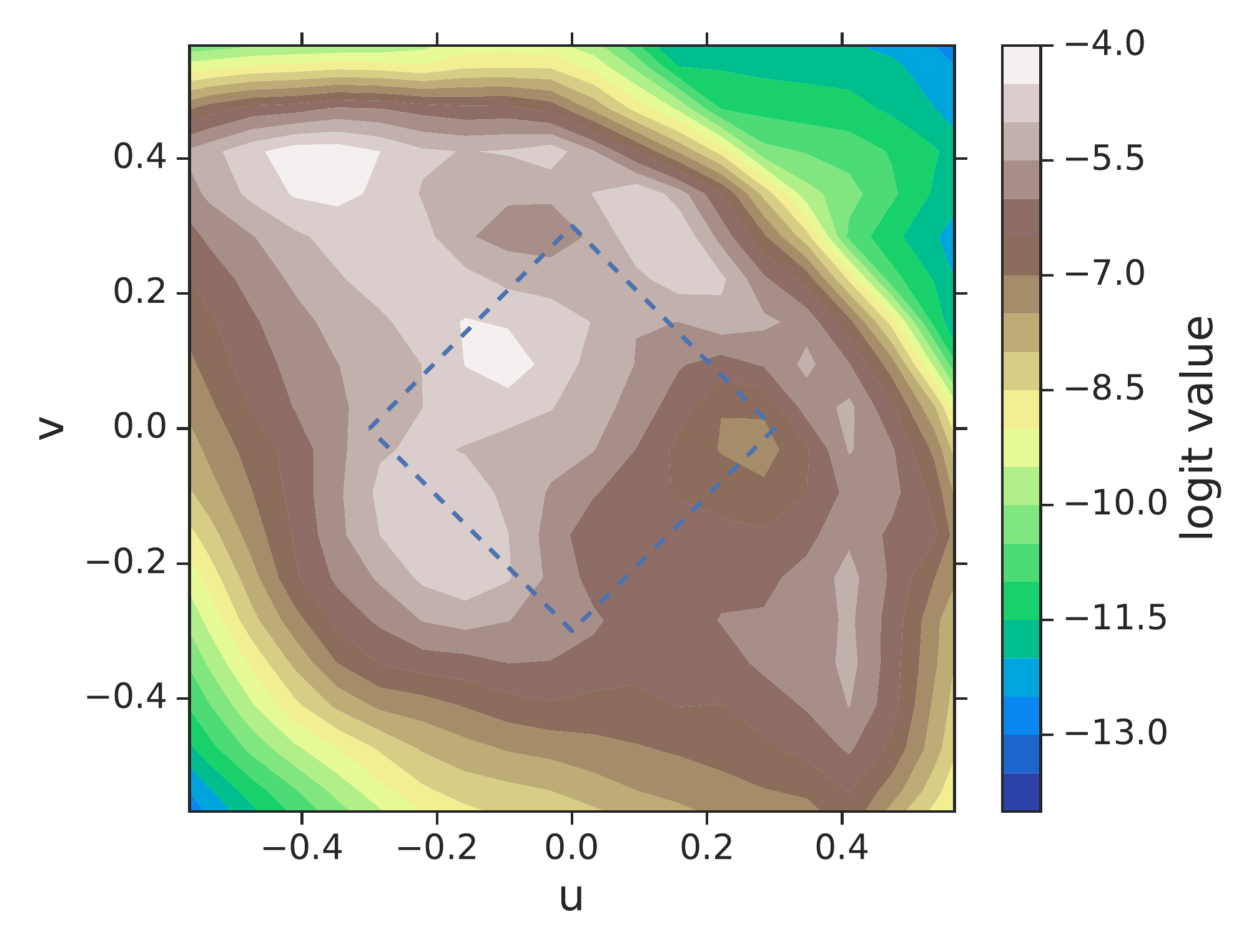}
\caption{$z_1$ on \citeauthor{madry_towards_2017}}
\end{subfigure} \\

\begin{subfigure}{0.32\textwidth}
\centering
\includegraphics[width=\linewidth]{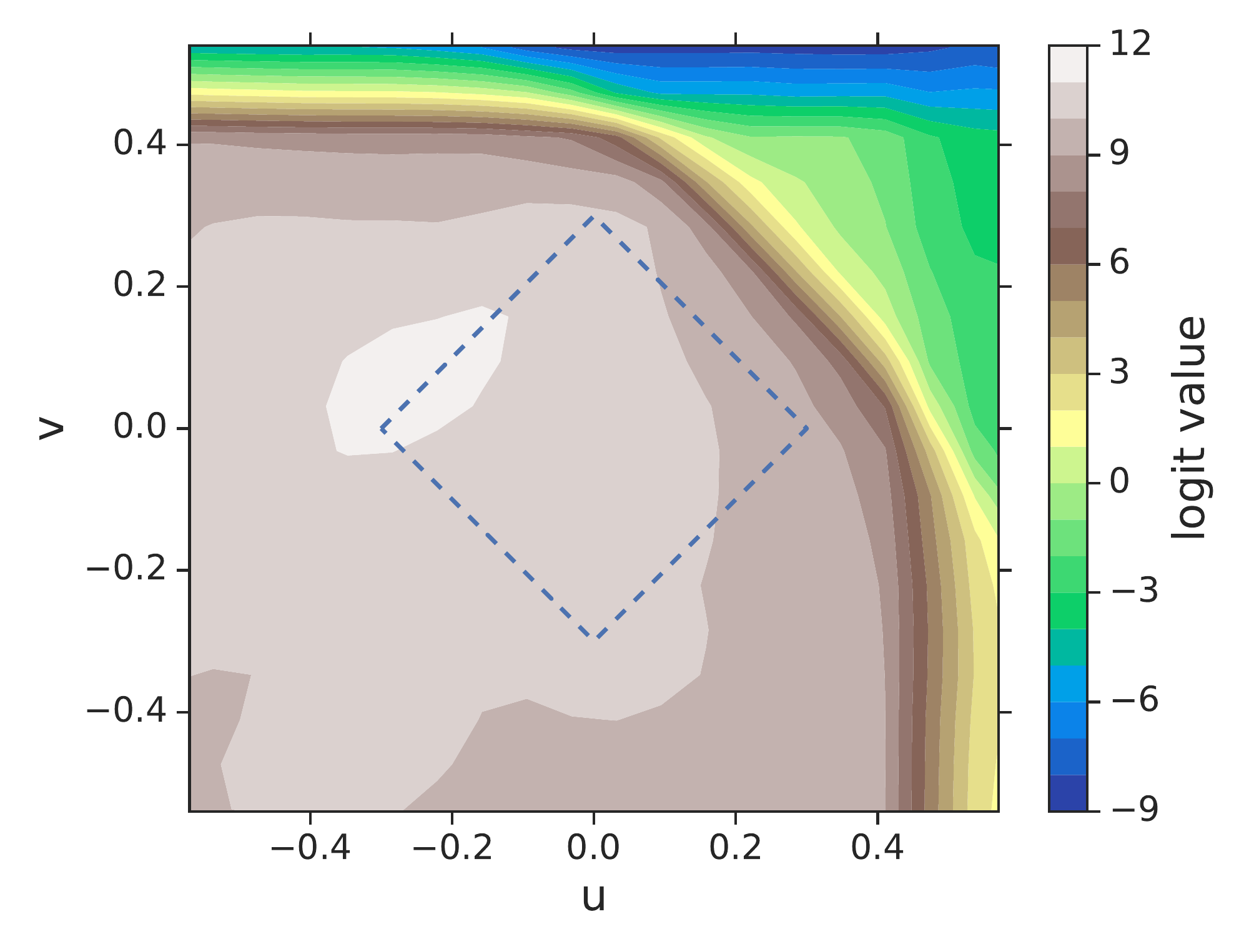}
\caption{$z_7$ on our model}
\end{subfigure}
\begin{subfigure}{0.32\textwidth}
\centering
\includegraphics[width=\linewidth]{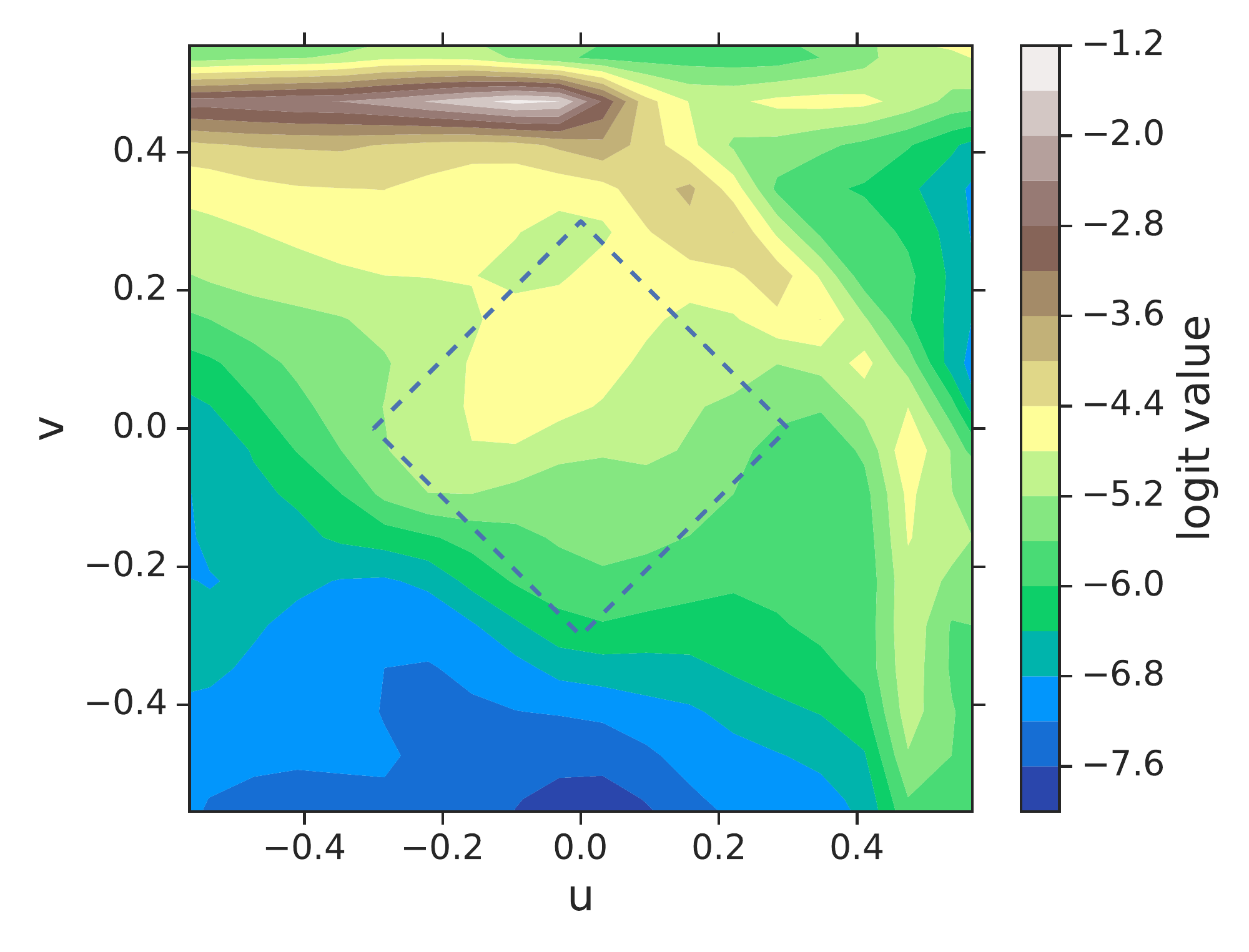}
\caption{$z_0$ on our model}
\end{subfigure}
\begin{subfigure}{0.32\textwidth}
\centering
\includegraphics[width=\linewidth]{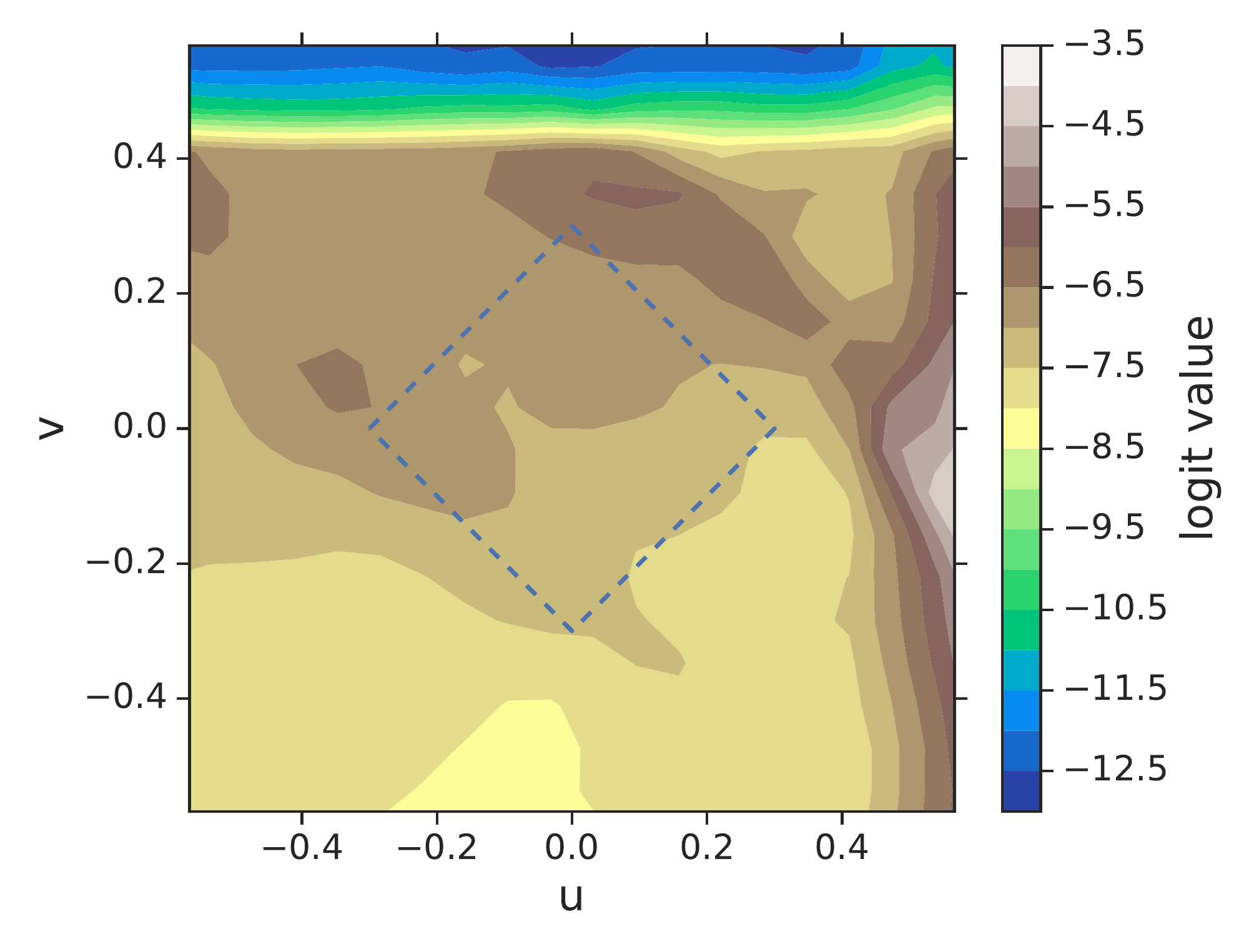}
\caption{$z_1$ on our model}
\end{subfigure} \\

\begin{subfigure}{0.32\textwidth}
\centering
\includegraphics[width=\linewidth]{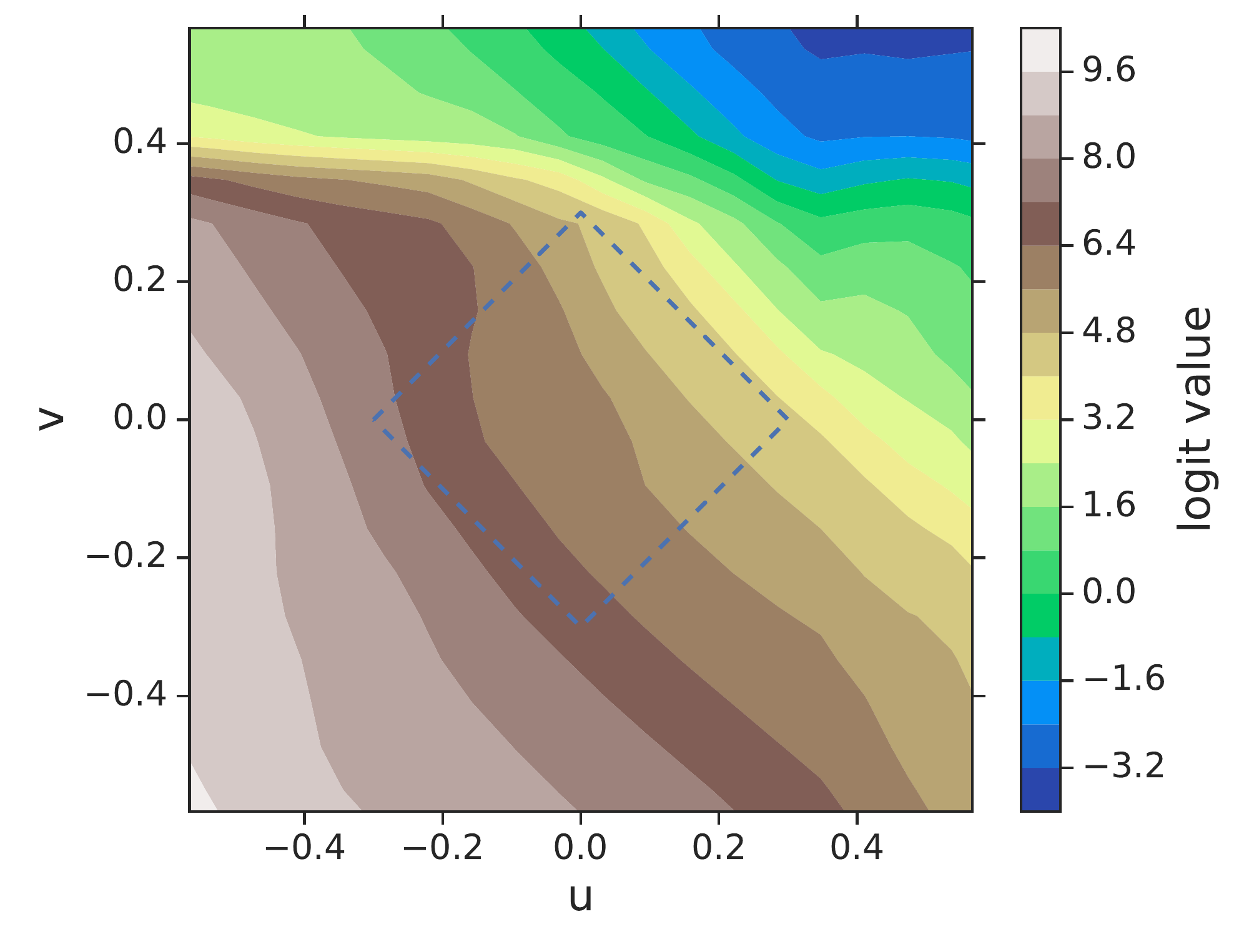}
\caption{$z_7$ on our small model}
\end{subfigure}
\begin{subfigure}{0.32\textwidth}
\centering
\includegraphics[width=\linewidth]{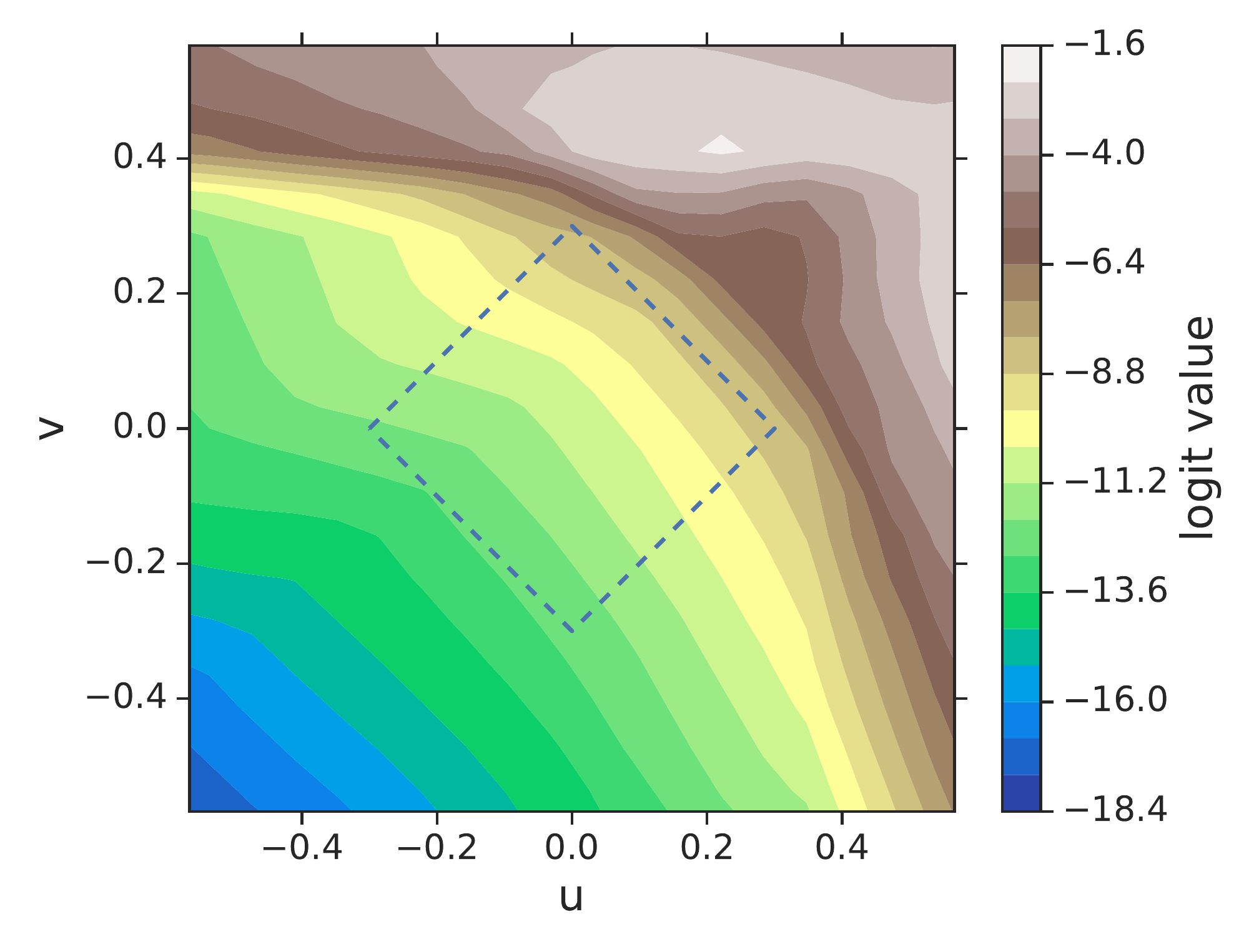}
\caption{$z_0$ on our small model}
\end{subfigure}
\begin{subfigure}{0.32\textwidth}
\centering
\includegraphics[width=\linewidth]{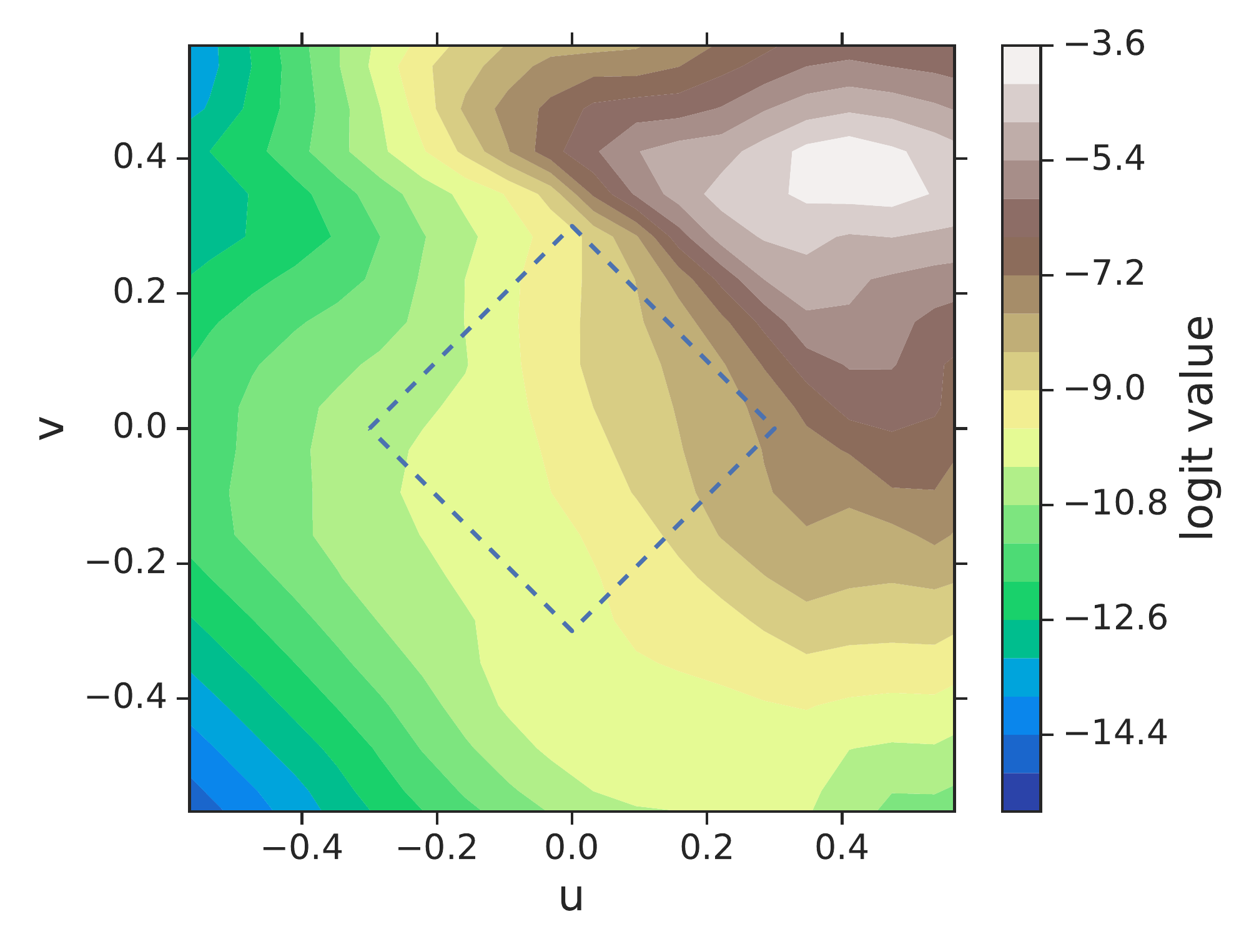}
\caption{$z_1$ on our small model}
\end{subfigure} \\
\caption{
Logit landscapes around the nominal image of a ``seven'' (first image from the \mnist test set).
It is generated by varying the input to the model, starting from the original input image toward either the worst attack found using PGD ($u$ direction) or a random Rademacher direction ($v$ direction).
Panels show the contour plots of different logits: from left to right, we have logits $z_7$, $z_0$ and $z_1$; from top to bottom, we have the \citeauthor{madry_towards_2017}'s model, our own adversarially trained model of similar size and our smaller adversarially trained model.
The diamond-shape represents the projected \linf ball of size $\epsilon = 0.3$ around the nominal image.
We observe that within this adversarial input set, the smaller model tends to behave more linearly than the other two models.
}
\label{fig:landscapes_mnist}
\end{figure*}

\begin{figure*}[t]
\centering
\begin{subfigure}{0.32\textwidth}
\centering
\includegraphics[width=\linewidth]{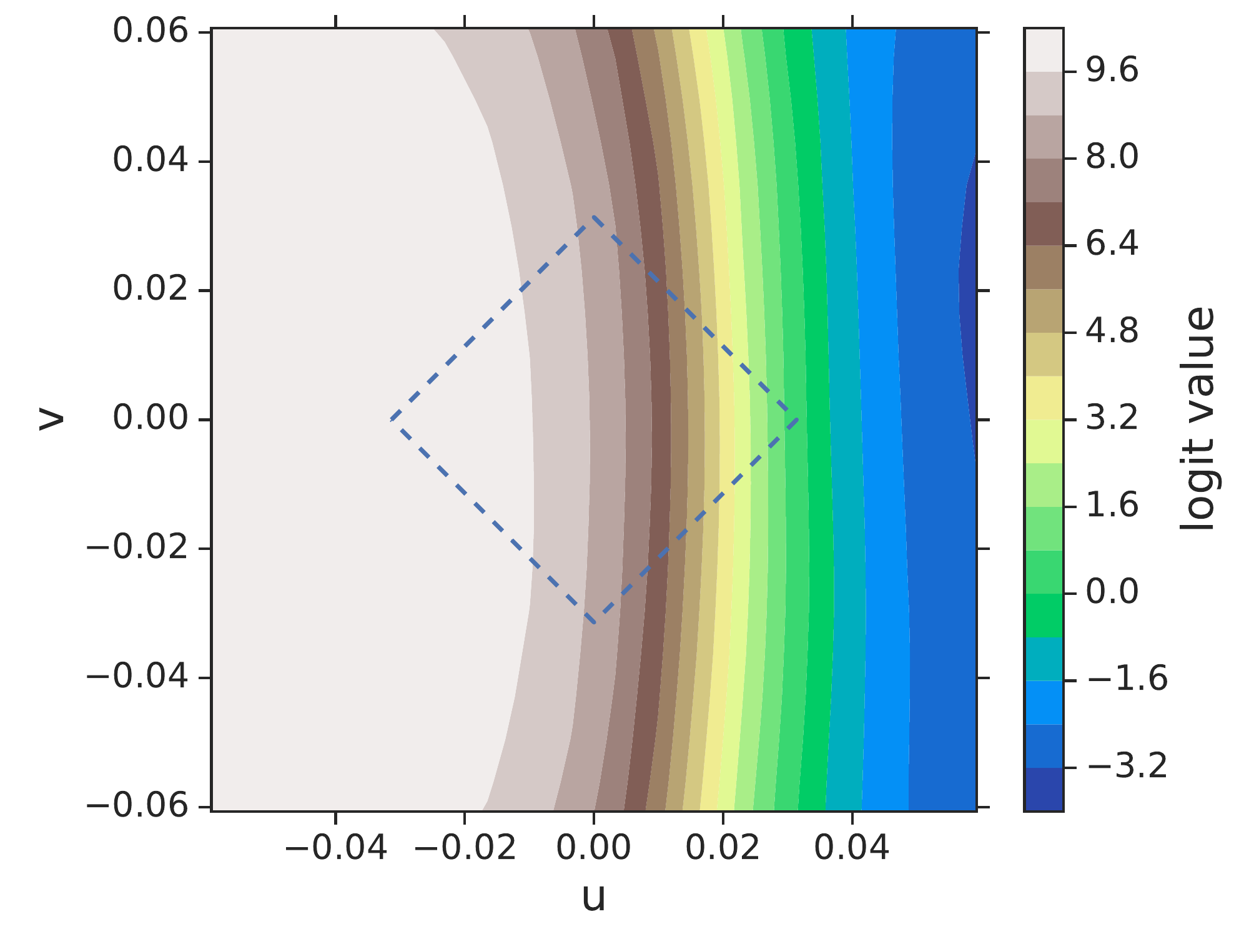}
\caption{$z_9$ on \citeauthor{madry_towards_2017}}
\end{subfigure}
\begin{subfigure}{0.32\textwidth}
\centering
\includegraphics[width=\linewidth]{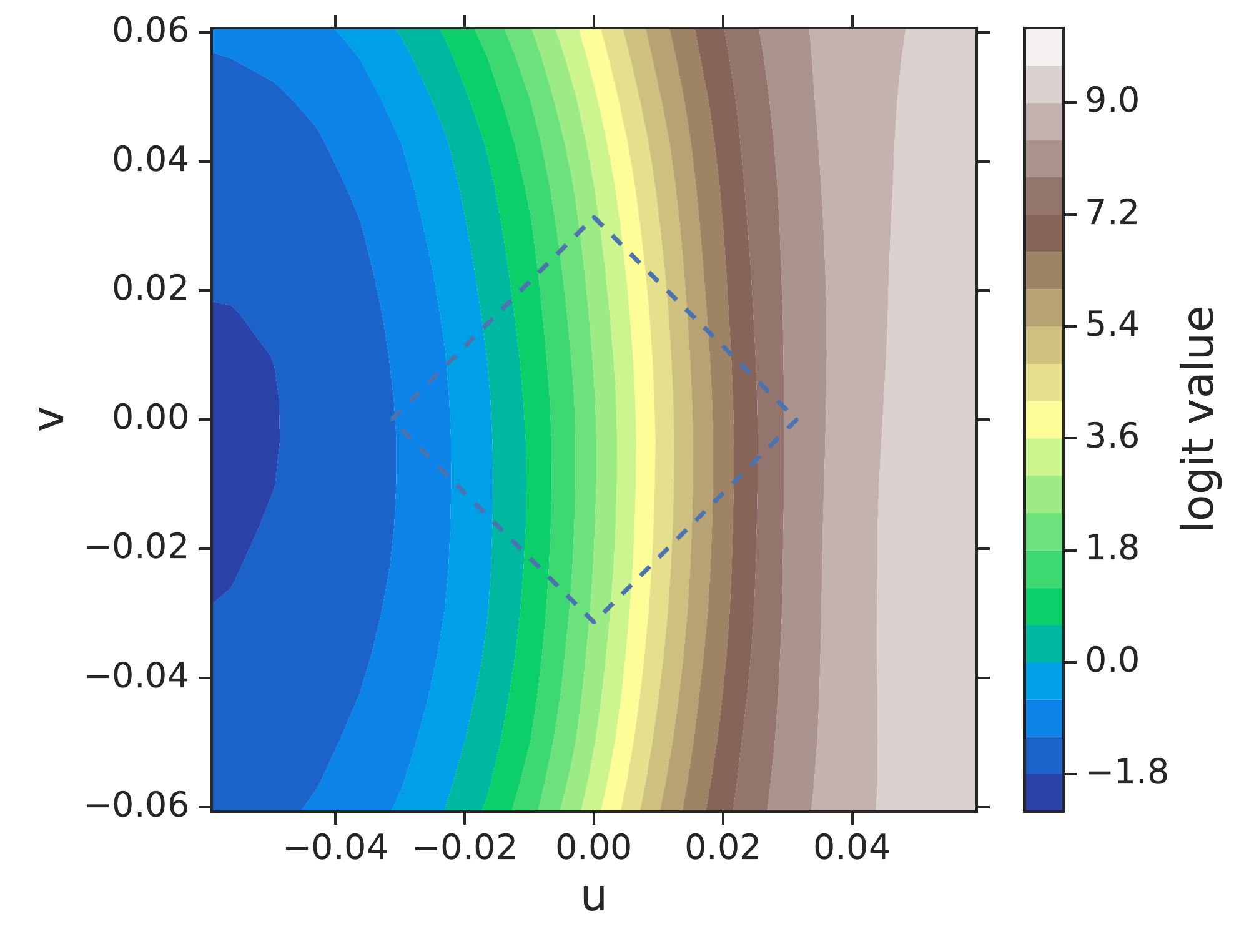}
\caption{$z_0$ on \citeauthor{madry_towards_2017}}
\end{subfigure}
\begin{subfigure}{0.32\textwidth}
\centering
\includegraphics[width=\linewidth]{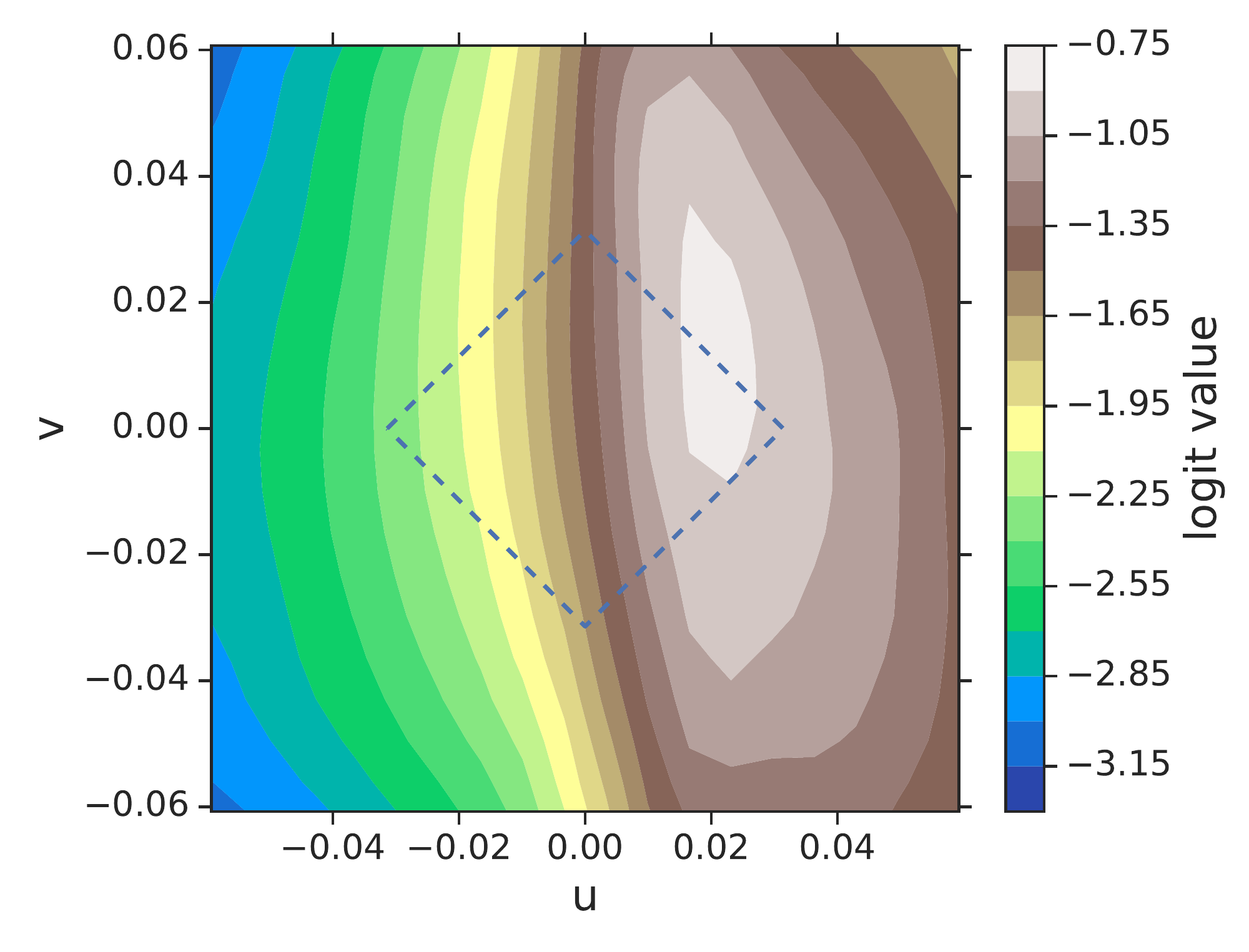}
\caption{$z_1$ on \citeauthor{madry_towards_2017}}
\end{subfigure} \\

\begin{subfigure}{0.32\textwidth}
\centering
\includegraphics[width=\linewidth]{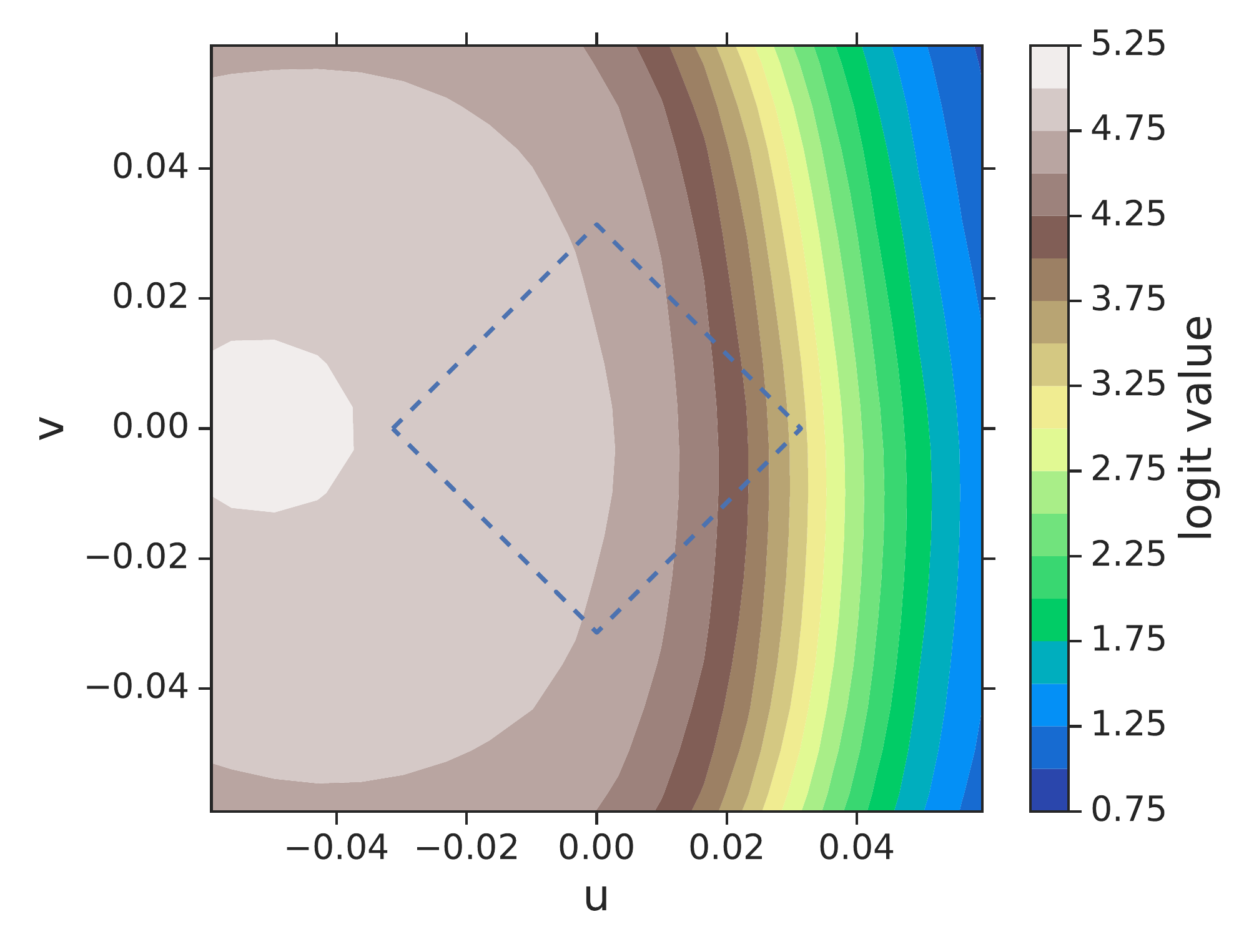}
\caption{$z_9$ on \citeauthor{zhang2019theoretically}}
\end{subfigure}
\begin{subfigure}{0.32\textwidth}
\centering
\includegraphics[width=\linewidth]{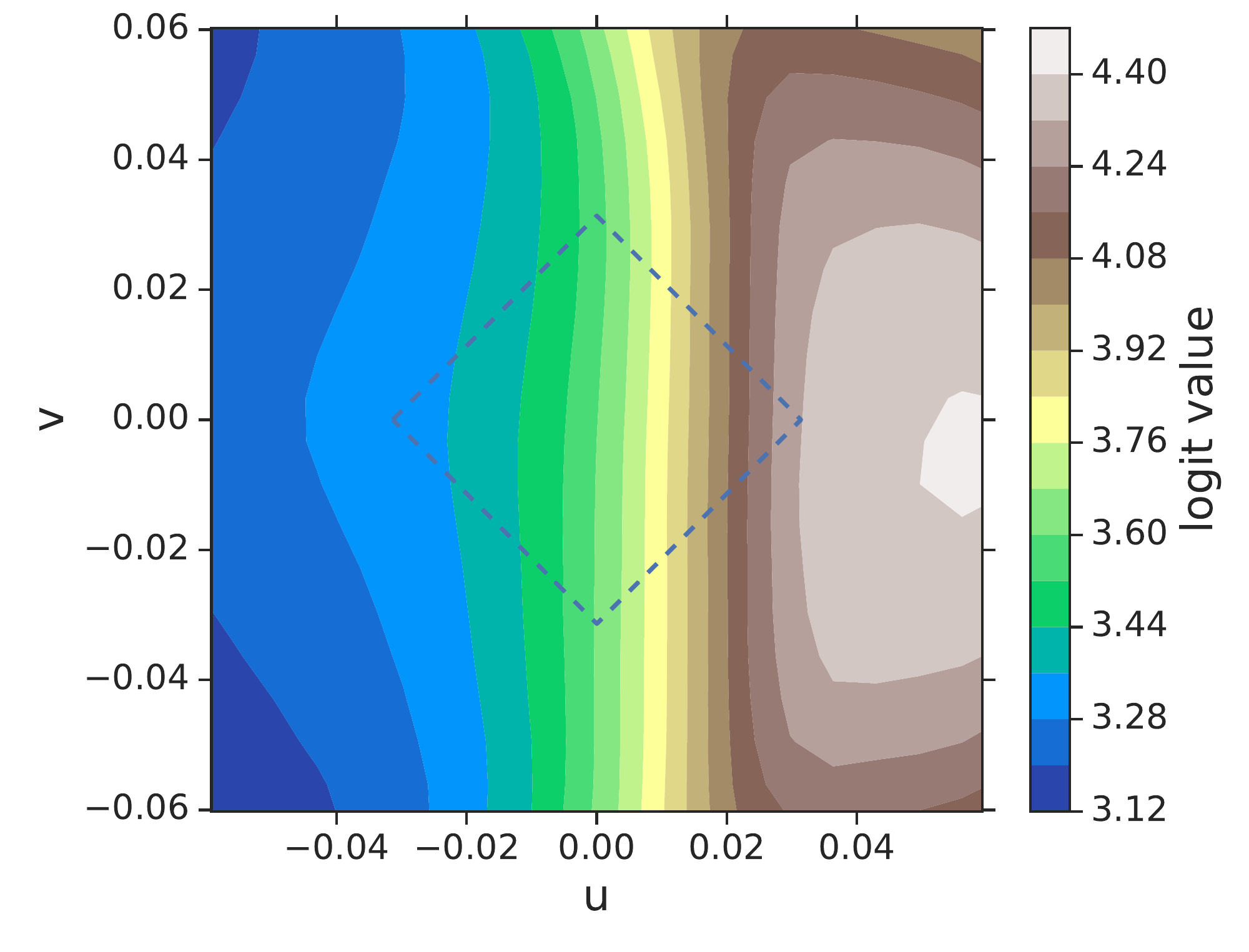}
\caption{$z_0$ on \citeauthor{zhang2019theoretically}}
\end{subfigure}
\begin{subfigure}{0.32\textwidth}
\centering
\includegraphics[width=\linewidth]{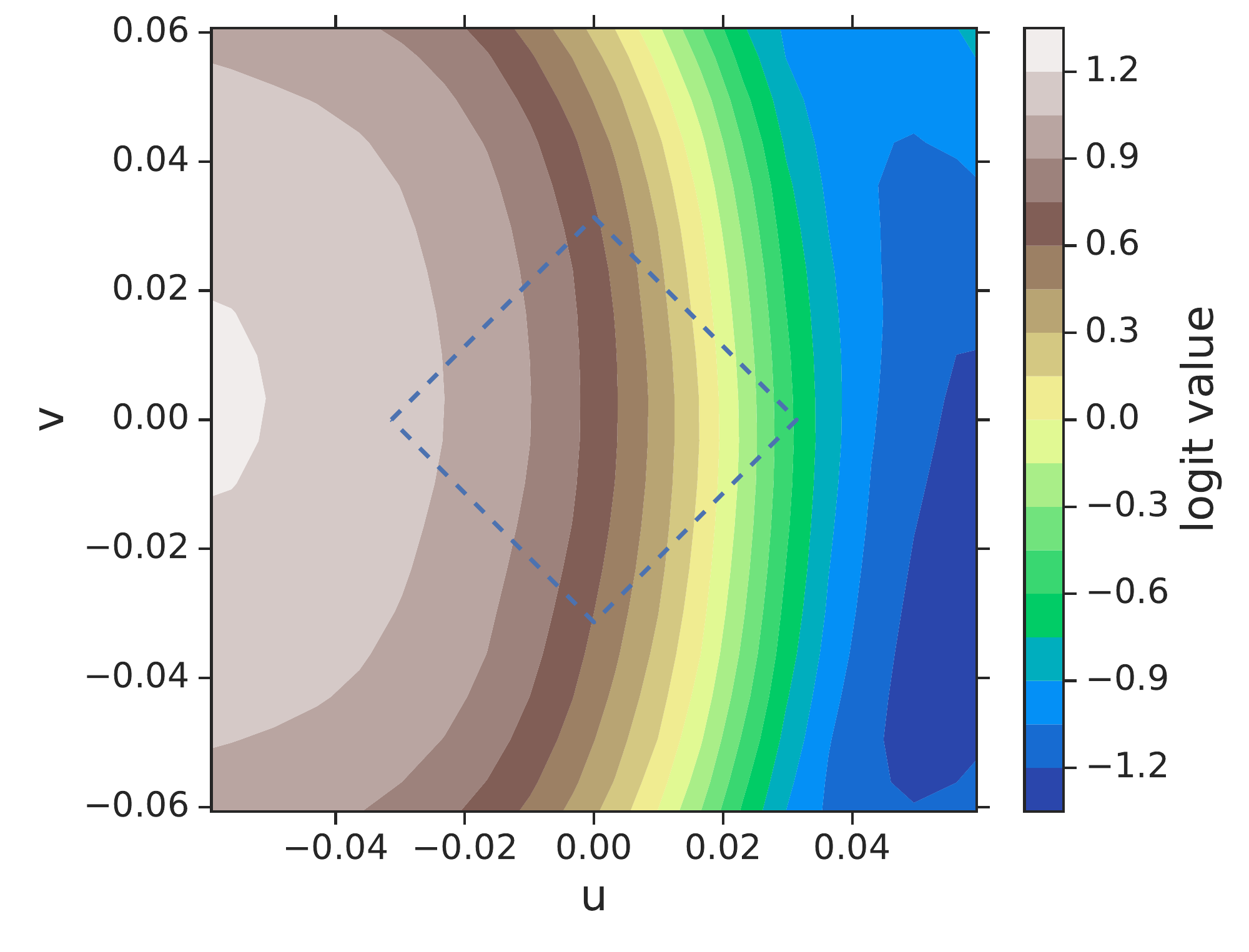}
\caption{$z_1$ on \citeauthor{zhang2019theoretically}}
\end{subfigure} \\

\begin{subfigure}{0.32\textwidth}
\centering
\includegraphics[width=\linewidth]{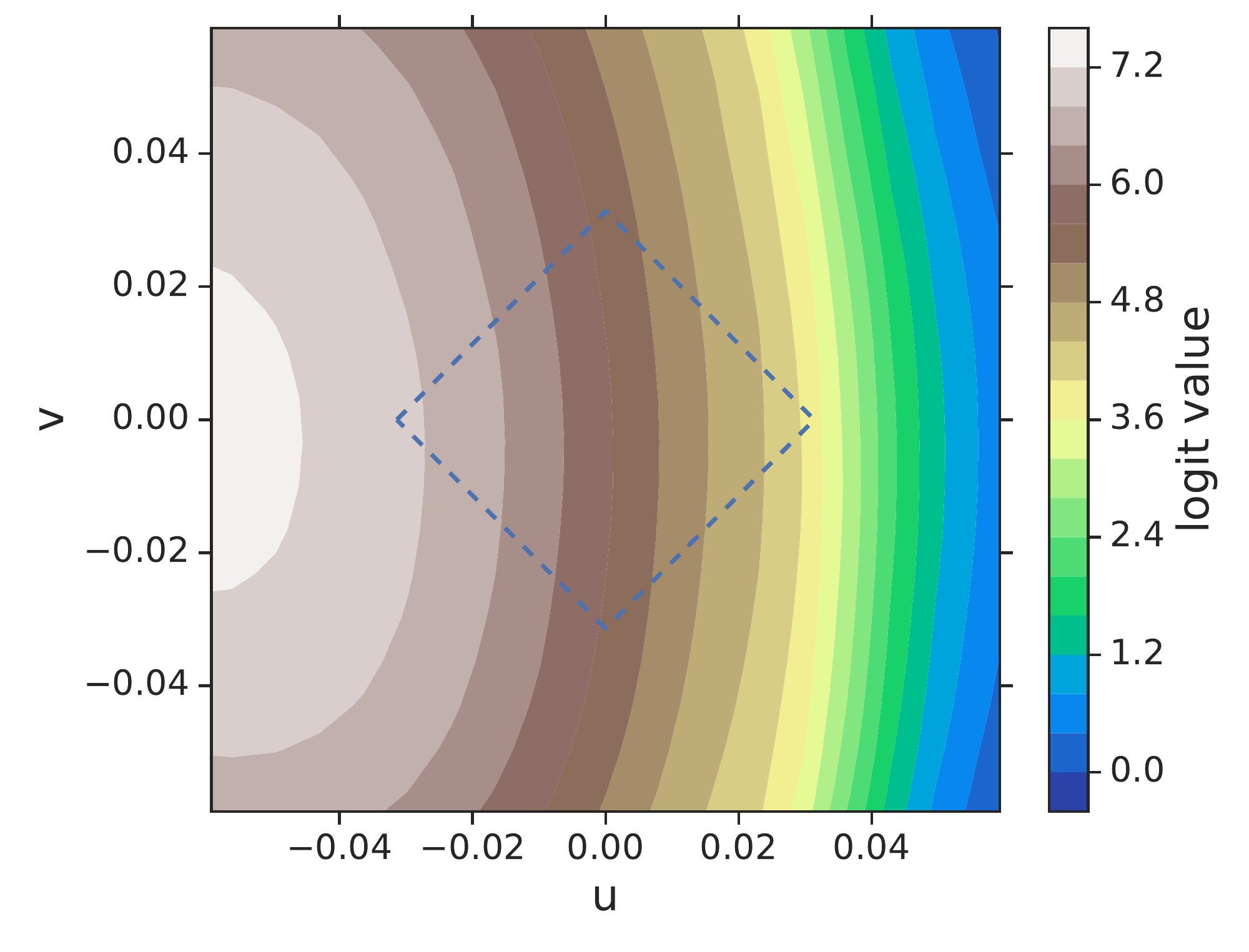}
\caption{$z_9$ on our model}
\end{subfigure}
\begin{subfigure}{0.32\textwidth}
\centering
\includegraphics[width=\linewidth]{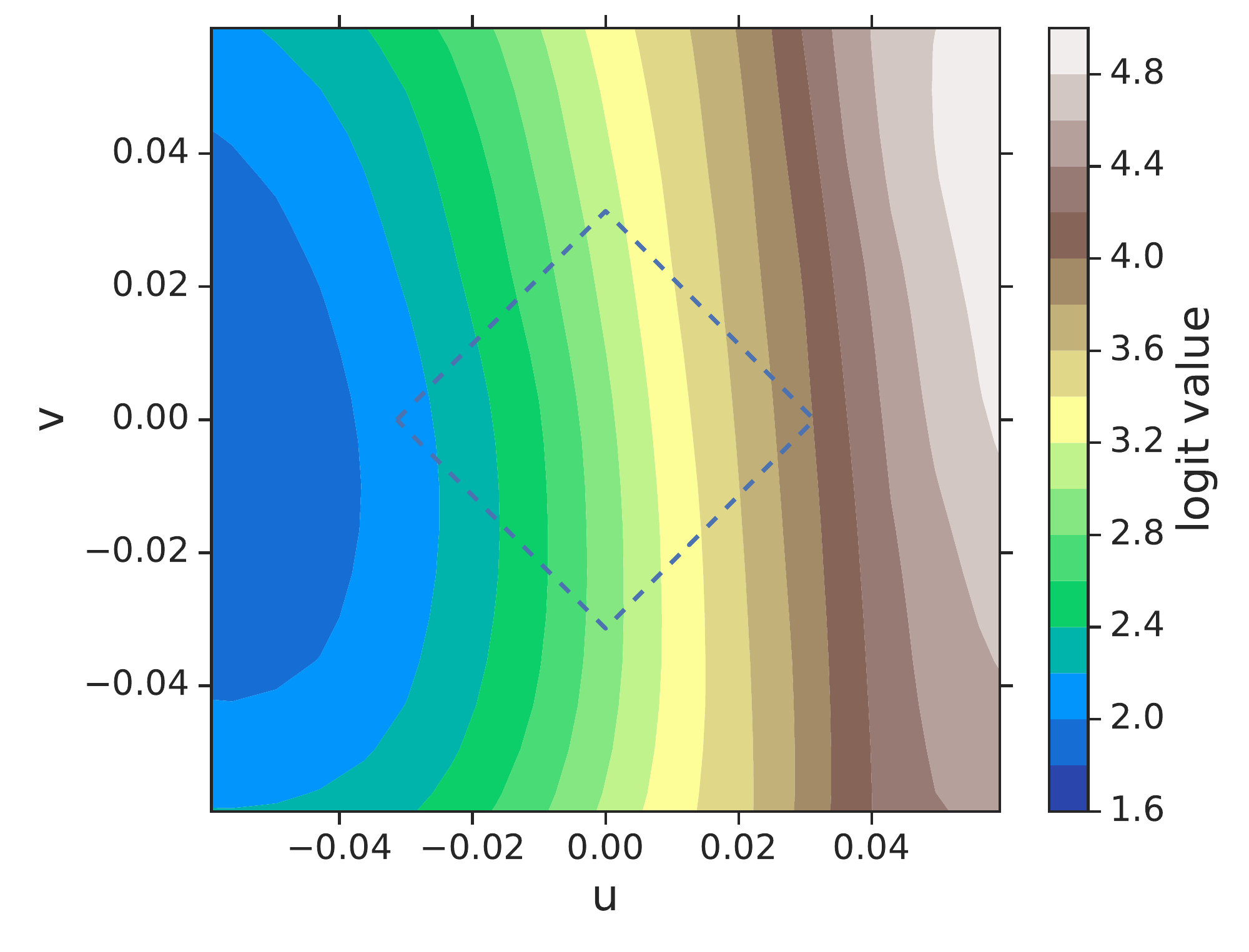}
\caption{$z_0$ on our model}
\end{subfigure}
\begin{subfigure}{0.32\textwidth}
\centering
\includegraphics[width=\linewidth]{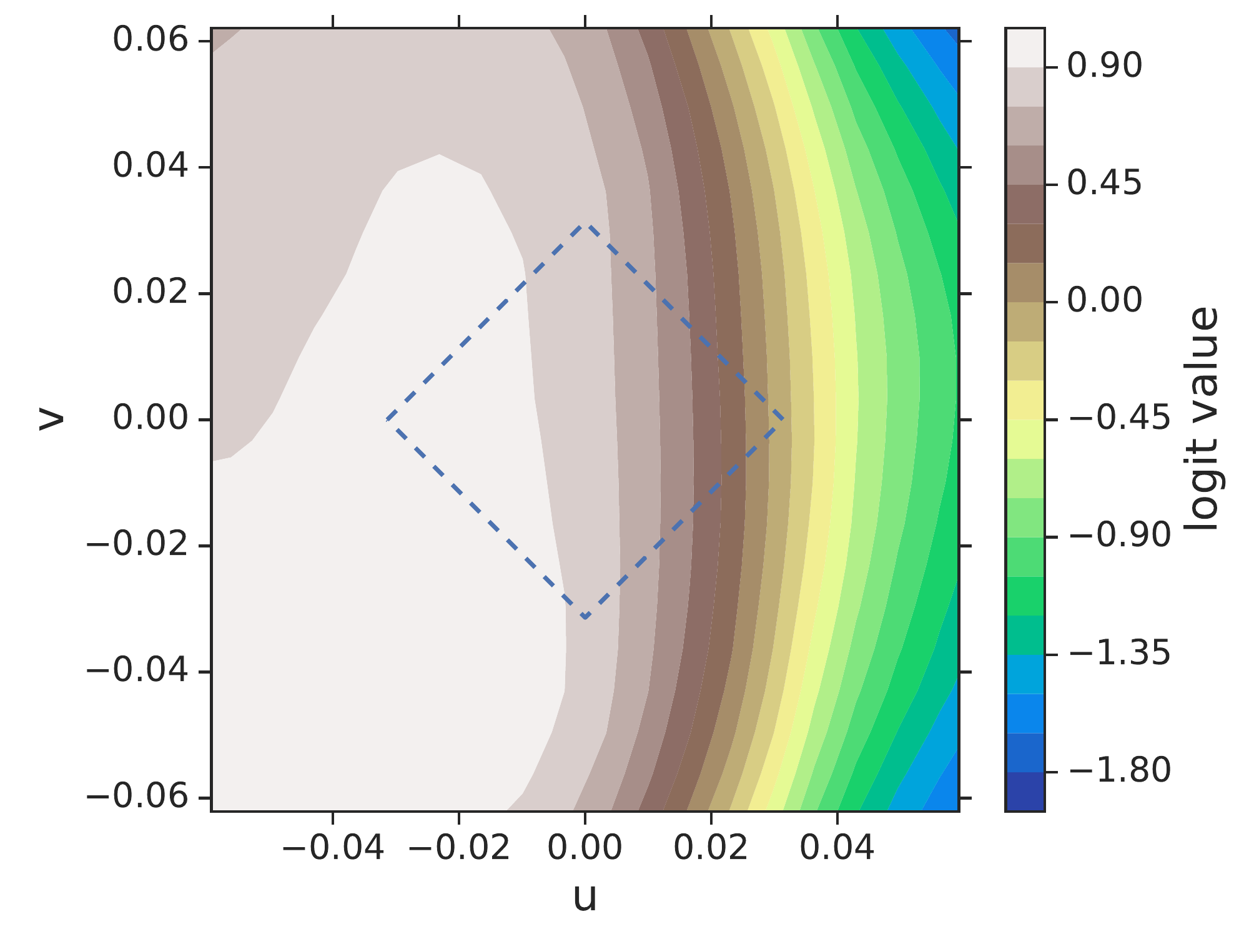}
\caption{$z_1$ on our model}
\end{subfigure} \\

\caption{
Logit landscapes around the nominal image of a ``truck'' (first image from the \cifar test set).
It is generated by varying the input to the model, starting from the original input image toward either the worst attack found using PGD ($u$ direction) or a random Rademacher direction ($v$ direction).
Panels show the contour plots of different logits: from left to right, we have logits $z_9$, $z_0$ and $z_1$; from top to bottom, we have the \citeauthor{madry_towards_2017}'s model, \citeauthor{zhang2019theoretically}'s model and our own adversarially trained model of similar size.
The diamond-shape represents the projected \linf ball of size $\epsilon = 8/255$ around the nominal image.
We observe that within this adversarial input set, all models tend to behave linearly.
}
\label{fig:landscapes_cifar}
\end{figure*}

\section{Reproducibility}
The code relevant to this publication is available as part of the Interval Bound Propagation library at \url{https://github.com/deepmind/interval-bound-propagation} under \path{src/attacks.py}.

\end{document}